\documentclass[]{fairmeta}
\usepackage{amsfonts}
\usepackage{amssymb}
\usepackage[english]{babel}
\usepackage{amsthm}
\usepackage{amsmath}
\usepackage{array}
\usepackage{multirow}
\usepackage{subcaption}
\usepackage{booktabs}
\usepackage{tabularx}
\usepackage{xcolor}     % For text color

\newtheorem{theorem}{Theorem}
\newtheorem{lemma}[theorem]{Lemma}

\title{Reward Models Are Secretly Value Functions: Temporally Coherent Reward Modeling}

\author[1]{Alex Nikulkov}

\affiliation[1]{AI at Meta}

\abstract{Reward models in RLHF are trained to score only the final token of a response---a choice that discards rich signal from every intermediate position and produces models whose token-level outputs are noise.
We argue this is a missed opportunity: a well-trained reward model's output at any token should represent the conditional expectation of the final reward given the response so far.
We introduce Temporally Coherent Reward Modeling (TCRM), which induces this property via two regularization terms on top of the standard Bradley-Terry loss, with minimizers provably equal to conditional expectations.
The regularizers correspond to Monte Carlo and TD value-learning objectives, establishing a direct connection to RL value functions.
TCRM requires zero changes to architecture, data, or inference, yet unlocks three capabilities from one principle: interpretable token-level reward trajectories (middle-token pairwise accuracy improved from 50\% to 88.9\%, final-token accuracy preserved); state-of-the-art PRM performance on ProcessBench (44.9\% average F1) among models trained only on outcome data; and unified reward/value modeling in PPO, reducing peak GPU memory by 27\% and step time by 19\% with matching LLM quality.
}

\date{\today}
\correspondence{Alex Nikulkov at \email{alexnik@meta.com}}

\begin{document}

\maketitle

\section{Introduction}
\label{section:intro}

Reinforcement Learning from Human Feedback (RLHF) has become the cornerstone for aligning Large Language Models (LLMs) with human values and intentions \cite{christiano2017rlhf,ouyang2022instructgpt,bai2022anthropichh,dubey2024llama3,liu2024deepseekv3}. At the heart of this process lies the reward model (RM), a crucial component used to evaluate the quality of a response to a prompt. In the most typical RLHF setup it serves as a proxy for human preference, guiding the LLM's output distribution toward generating more helpful and harmless content.

The standard approach to training these models is effective yet surprisingly inefficient. Typically, a decoder-only transformer architecture is fine-tuned using a pairwise preference loss, such as the Bradley-Terry model \cite{bradley1952bradleyterry}. This loss is calculated based only on the model's output at the final token of each sequence (e.g., the <EOS> token). The rich information contained in the outputs at all intermediate token positions is entirely discarded during training. This results in reward models whose intermediate predictions are noisy, uninterpretable, and uncorrelated with the quality of the partial responses being generated.

In this work, we introduce Temporally Coherent Reward Modeling (TCRM), a new training paradigm designed to make the output of a reward model at every token meaningful. Our central thesis is that the reward model's output at any given token should represent the conditional expectation of the final reward, given the partial response generated so far. By inducing this property, the reward trajectory becomes a smooth and interpretable signal that reflects the evolving quality of the response as each new token is generated.

We achieve this coherence by augmenting the standard Bradley-Terry loss with two regularization terms whose minimizers are provably conditional expectations of the final reward (Lemmas \ref{lem:lookahead_cond_exp} and \ref{lem:smooth_cond_exp}). This single principle---that every intermediate token should predict the final reward---\textbf{requires no changes to model architecture, training data, or inference}, yet unlocks three distinct capabilities:

\begin{enumerate}
    \item \textbf{Interpretability}. TCRM transforms intermediate outputs from noise into a predictive signal, raising middle-token pairwise accuracy from near-chance 50\% to up to 88.9\% while preserving final-token accuracy. This enables token-level credit assignment for debugging and analysis.
    \item \textbf{Process-level evaluation} without step labels. TCRM functions as a highly effective Process Reward Model, achieving 44.9\% average F1 on ProcessBench without requiring the step-level supervision that standard PRMs depend on. 
    \item Unified reward and \textbf{value modeling} in PPO. Because our regularizers correspond structurally to Monte Carlo and TD value-learning objectives, the same TCRM can serve as both reward and value model in PPO, reducing peak GPU memory by 27\% and training step time by 19\% with no loss in LLM quality.
\end{enumerate}

\section{Preliminaries}
\label{section:prelim}

The standard RLHF reward modeling recipe includes a decoder-only transformer model architecture and a pairwise Bradley-Terry loss. The loss is applied to the model outputs which correspond to the final token of each sequence (usually the EOS token). The model outputs from all tokens except last are discarded during training and inference.

Notation:
\begin{itemize}
    \item $x, y$ - prompt and response, respectively. Each of them consists of a sequence of individual tokens: $x=(x_0, x_1, \cdots, x_N), y=(y_0, y_1, \cdots, y_K)$. $y_{0..k}=(y_0, y_1, \cdots, y_k)$ is a partial response, up to the $k$-th token. We assume that the final response  token is always the EOS token: $y_K=$"<EOS>".
    \item $r()$ - decoder-only transformer reward model. $r(x,y)$ is the model output at the final response token, $r(x,y_{0..k})$ is the model output at the $k$-th response token. The model is parameterized by $\theta$, but we drop $\theta$ from the notation for brevity.
    \item $y^w, y^l$ - the preferred (winner) and rejected (loser) responses to the same prompt. The preference labels are usually generated by human annotators and/or a carefully prompted LLM.
\end{itemize}

The Bradley-Terry loss for a single pair of responses is $L_{BT}(x,y^w,y^l)=-\log(\sigma(r(x, y^w) - r(x, y^l)))$. The reward model is trained by minimizing this loss on a fixed dataset.

\section{Intermediate Reward Model Outputs}
\label{section:des}

Meaningful reward model outputs at intermediate tokens can improve the model quality, interpretability, and usefulness. We would like these outputs to represent conditional expectation of the final reward, given the partial response tokens:

\[
r(x, y_{0..k})=E[r(x, y)|x, y_{0..k}]
\]

Ideally we would model $E[R(x, y)|x, y_{0..k}]$, where $R$ is the unobserved ground-truth reward. In pairwise RM training we do not observe $R$, so we use the final model score $r(x,y)$ as a proxy target. Under a standard unbiasedness assumption, this gives the same conditional-expectation form for intermediate predictions.

The expectation is over future tokens $y_{k+1},\cdots,y_K$, so it depends on both the observed prefix and the policy that generates the continuation. In offline training, this continuation distribution is the training-data distribution; during online RL it shifts with the policy. Empirically (Section \ref{sec:cross_dataset}), TCRM's intermediate-token accuracy is no more sensitive to this shift than final-token accuracy and even a frozen TCRM can effectively function as a value model in PPO with no extra fine-tuning (Section \ref{sec:rl}).

Note that conditional expectation captures expected final outcome after continuation, not intrinsic prefix quality. Despite this caveat, intermediate outputs are interpretable in our experiments.

The decoder-only transformer architecture used in LLMs is uniquely suited to represent the conditional expectation due to causal masking applied during model training. As shown in Figure \ref{fig:decoder_attention}, the decoder output at an intermediate token depends only on the tokens at the current and previous positions, which are exactly the tokens on which the expectation is conditioned. This enables us to use a single model for both intermediate and final reward predictions, since they differ only in the sequence of tokens on which they are conditioned.

\begin{figure}
    \centering
    \includegraphics[width=0.5\linewidth]{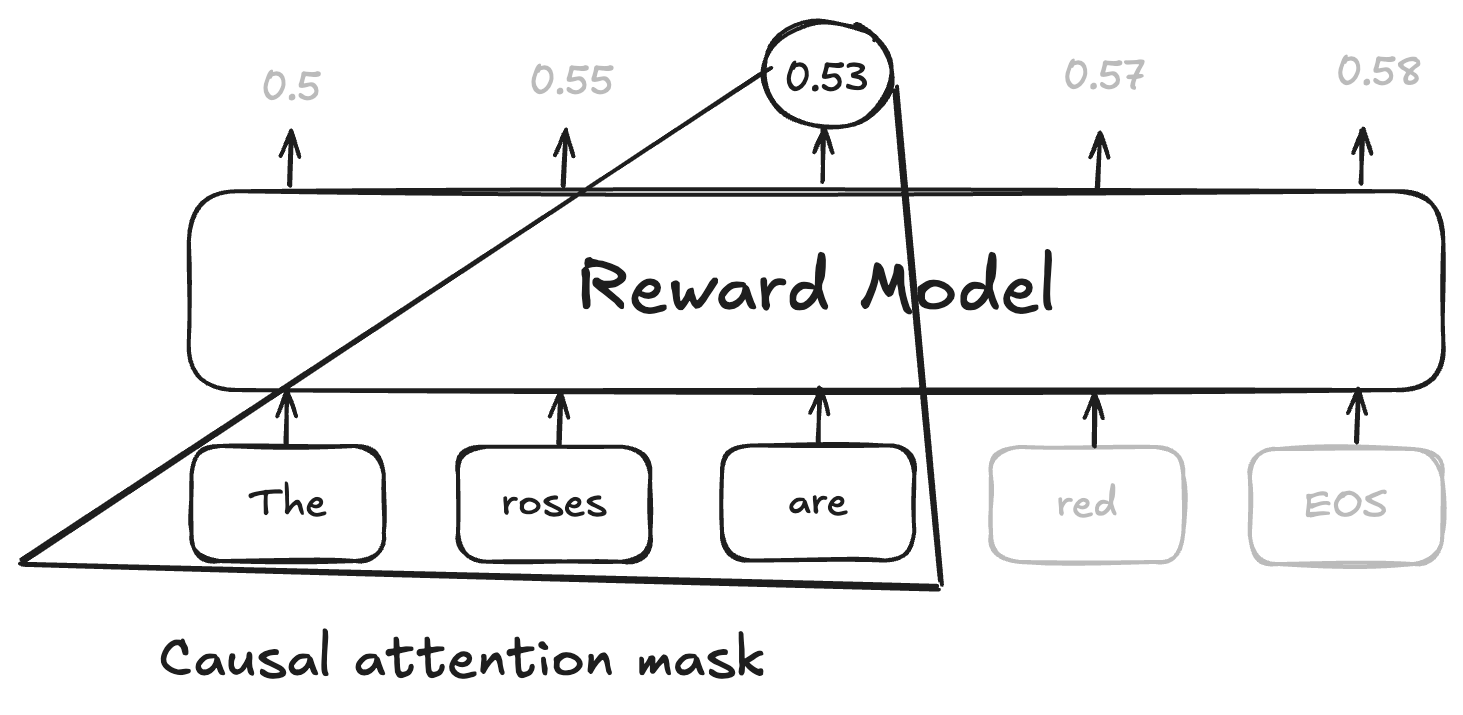}
    \caption{Causal attention mask of a decoder-based reward model.}
    \label{fig:decoder_attention}
\end{figure}

By comparing the model outputs at two adjacent token positions, we can measure the incremental impact of each token on response quality: $d(y_k|x,y_{0..k-1}) = r(x,y_{0..k}) - r(x,y_{0..k-1})$. The final reward model prediction can be decomposed into a sum of token-level incremental contributions: $r(x,y)=r(x,y_0) + \sum_{k=1}^K d(y_k|x,y_{0..k-1})$. Note that the reward of the first output token $r(x,y_0)$ is treated differently, but we can unify its treatment by assuming that the reward of an empty response is 0: $r(x,\{\})=0$, $d(y_0|x,\{\})=r(x,y_0) - r(x,\{\})=r(x,y_0)$.

There are several benefits from having coherent intermediate reward model outputs:
\begin{enumerate}
    \item \textbf{Interpretability}. $d(y_k|x,y_{0..k-1})$ helps explain final reward predictions and diagnose model behavior. For example, a safety-focused reward model should assign negative increments to unsafe content; unexpected increments can reveal training issues.
    \item \textbf{Consistency across tokens}. Standard reward models are largely incoherent at intermediate tokens and only become meaningful near EOS (Figures \ref{fig:fib_baseline_correct}, \ref{fig:fib_baseline_wrong}, \ref{fig:fib_baseline_extra}). Enforcing consistency enables scoring of partial responses and encourages useful computation throughout the sequence, which might also improve final-token predictions.
    \item \textbf{High-resolution signal for RL}. Meaningful token-level rewards can be used to enable RL with higher granularity compared to a single outcome rewards. While true token-level reward signals might be too noisy, we can aggregate them into longer chunks (e.g. sentence-level or step-level rewards) to improve the signal-to-noise ratio. Token-level rewards can also be used as the value model in PPO training (a consequence of the structural correspondence established in Section~\ref{sec:value-learning-connection}).
\end{enumerate}

\section{Temporally Coherent Regularization}
\label{section:temp_coh_reg}

The standard Bradley-Terry reward modeling recipe produces reward models which lack temporal coherence (meaningless outputs at intermediate tokens), but we can induce the coherence by adding extra regularization terms to the loss. We propose 2 types of regularization terms:
\begin{enumerate}
    \item \textbf{Lookahead consistency}. This term penalizes large differences between model outputs at intermediate and final tokens; its MSE minimizer is the conditional expectation of the final reward given the current prefix (Lemma \ref{lem:lookahead_cond_exp}).
    \item \textbf{Smoothness}. This term encourages local trajectory consistency; under recursive minimization, each step becomes the conditional expectation of the next-step target (Lemma \ref{lem:smooth_cond_exp}).
\end{enumerate}

Crucially, these regularizers operate entirely on the existing reward model outputs---no new heads, auxiliary networks, or additional labels are introduced.

\subsection{Lookahead Consistency}
\label{sec:lookahead_reg}

Using $r(x, y_{0..k})=E[r(x, y)|x, y_{0..k}]$, we define an MSE loss between intermediate and final outputs. By Lemma \ref{lem:lookahead_cond_exp} (formal statement in Appendix), this loss is minimized by the conditional expectation, so it induces lookahead consistency.

The lookahead consistency regularization term is defined as:

\begin{equation}
L_{LA}(x,y) = \sum_{k=0}^{K-1} (r(x,y_{0..k}) - SG[r(x,y)])^2
\label{eq:lookahead_loss}
\end{equation}

To stabilize learning and avoid regressing the quality of the final reward predictions, we detach (stop-gradient) the final reward model output, similar to how the Bellman equation targets are detached in DQN to stabilize Q-learning. See Appendix \ref{sec:ablations} for comparison of loss with and without stop-gradient.

\subsection{Smoothness}
\label{sec:reg_smoothness}

Most tokens in a typical LLM response are just filler tokens and don't substantially impact the quality of the response. We'd like the reward model outputs at adjacent tokens to be similar in the absence of substantial new information in the token. Since reward model training data usually doesn't highlight which parts of the response are substantial, we can add a uniform regularization term which penalizes differences in reward scores of adjacent response tokens regardless of whether there is a substantial information in the token. For non-substantial tokens this should create an incentive for the scores to remain close. For tokens with substantial information there could be a tradeoff between accuracy and trajectory smoothness, but we can manage this tradeoff by tuning the weight on the regularization coefficients. Empirically, we find that smoothness loss helps regularize the model and reduces the noise in token-level predictions.

Lemma \ref{lem:smooth_cond_exp} (formal statement in Appendix) shows that recursive minimization of smoothness yields a Doob martingale: each step is the conditional expectation of the next-step target (hence of final reward). Joint minimization of all squared differences instead gives $X_t^*=\frac{X_{t-1} + E\big[X_{t+1} \mid \mathcal{F}_t \big]}{2}$, so we apply stop-gradient to the next-step reward to better match the recursive setting.

The smoothness regularization term for a single response $y$ of length $K$ to the prompt $x$ is defined as the sum of squared differences along the trajectory:

\begin{equation}
L_{sm}(x,y) = \sum_{k=1}^K (r(x,y_{0..k-1}) - SG[r(x,y_{0..k})])^2
\label{eq:smoothness_loss}
\end{equation}

\subsection{Overall Loss}
\label{sec:overall_loss}

Since each of the losses described in previous subsections has its own pros and cons, we combine them into a weighted sum to amplify their benefits and minimize the downsides:

\begin{multline}
L_{overall}(x,y^w,y^l) = L_{BT}(x,y^w,y^l) + a_{sm}(L_{sm}(x,y^w) + L_{sm}(x,y^l)) + \\
+ a_{LA}(L_{LA}(x,y^w) + L_{LA}(x,y^l))
\end{multline}

\subsection{Connection to Value Learning}
\label{sec:value-learning-connection}

The two regularizers introduced above are not arbitrary choices: they correspond 
structurally to the two canonical objectives used to train value functions in 
reinforcement learning. Making this correspondence explicit clarifies both why 
TCRM works and why the same model can later serve as a value model in PPO 
(Section~\ref{sec:rl}).

Consider a value model $V(x, y_{0..k})$ trained to predict the return from prefix 
$y_{0..k}$. In the RLHF setting, rewards are sparse: the only non-zero reward 
arrives at the final token, so with $\gamma = 1$ the return equals $r(x, y)$ for 
every prefix. The two standard value-learning losses then take the following forms:

\paragraph{Monte Carlo (full-return) target.}
\begin{equation}
\mathcal{L}_{\text{VM-MC}} = \bigl(V(x, y_{0..k}) - r(x, y)\bigr)^2
\end{equation}
Each prefix is regressed directly against the observed final reward. Compare this 
to our lookahead consistency loss \eqref{eq:lookahead_loss}. The two are structurally identical: the reward model's own final-token output 
plays the role of the Monte Carlo target, with stop-gradient ensuring the target 
is not itself a learning signal.

\paragraph{Temporal Difference (bootstrapped) target.}
\begin{equation}
\mathcal{L}_{\text{VM-TD}} = \bigl(V(x, y_{0..k}) - \text{SG}[V(x, y_{0..k+1})]\bigr)^2
\end{equation}
(for non-terminal transitions; at the terminal token the target is the observed 
reward $r(x, y)$). Our smoothness loss \eqref{eq:smoothness_loss} has the same one-step bootstrapping 
structure: each intermediate output is regressed against the next-step output, recovering 
the Bellman consistency that defines a TD value update.

\section{Experiments}
\label{section:exp}

\subsection{Training Details}

We use the \textit{Skywork-Reward-Preference-80K-v0.2} \cite{liu2024skywork} binary preference dataset for model training and evaluation. Since the data are not split into train/test in the original dataset, we split them 90/10. The reward models were initialized with Qwen3 (6 sizes between 0.6B and 32B) and Llama 3.1 8B instruction-tuned checkpoints.
We also trained 2 other types of token-level reward models for comparison: ImplictPRM \cite{yuan2024free} and TC-$\lambda$ \cite{maystre2025incremental} on the same data. Additional training details are listed in Appendix \ref{section:add_train_det}.

\subsection{Metrics}

The following metrics are used to measure the benefits of temporally coherent reward modeling:
\begin{enumerate}
    \item \textbf{Final position pairwise accuracy}. The reward model must maintain high accuracy when evaluating the full response.
    \item \textbf{Middle position pairwise accuracy}. This is a measure of how well the reward model can predict which of the 2 responses is better by looking at just the first half of each response. Improvements in this metric would indicate that the model uses the information from partial responses more effectively.
    \item \textbf{Mean squared step delta}. This is a measure of the "jumpiness" of the reward model predictions at intermediate token positions. It is defined as $\frac{\sum_{k=1}^K (r(x,y_{0..k}) - r(x,y_{0..k-1}))^2}{K}$ and is related to the "smoothness" part of the loss function.
    \item \textbf{Mean squared final delta}. This metric measures how different the intermediate reward model outputs are compared to the final output. It is defined as $\frac{\sum_{k=1}^K (r(x,y_{0..k}) - r(x,y))^2}{K}$ and is related to the "lookahead consistency" part of the loss function.
\end{enumerate}

In our experiments middle position accuracy was the main metric to improve. Final position accuracy could not be improved substantially compared to the baseline model, but we treated it as a guardrail to avoid regressions. The "delta" metrics were secondary and were used mostly for diagnostics and analysis.

\subsection{Temporally Coherent Reward Models}

For each model size, we train up to 4 reward models starting from the same base checkpoint:
\begin{enumerate}
    \item \textbf{Baseline} reward model, using only the standard Bradley-Terry loss for the final-token reward model outputs. Intermediate outputs are ignored.
    \item \textbf{ImplicitPRM (DPO)} \cite{yuan2024free} reward model, trained with DPO-style objective.
    \item \textbf{TC-$\lambda$} \cite{maystre2025incremental} reward model, trained with temporal-consistency smoothing objective.
    \item \textbf{Temporally coherent} reward model, trained with the loss described in Section \ref{section:temp_coh_reg}.
\end{enumerate}

\begin{table}[h]
    \caption{Main reward-model results: accuracy and temporal coherence metrics}
    \label{tab:main_metrics}
    \centering
    \begin{tabular}{|c|c|c|c|c|c|}
    \hline
    \multirow{2}{*}{\textbf{Base Model}} & \multirow{2}{*}{\textbf{Reward Model Type}} & \multicolumn{2}{c|}{\textbf{Accuracy}} & \multicolumn{2}{c|}{\textbf{Mean Squared Delta}} \\ \cline{3-6} 
     & & \textbf{Final Token} & \textbf{Middle Token} & \textbf{Step} & \textbf{Final} \\ \hline
    \multirow{4}{*}{Qwen3 0.6B} & Baseline & 83.9\% & 48.0\% & 5.35 & 10.70 \\ \cline{2-6}
     & ImplicitPRM (DPO) & 78.2\% & \textbf{73.6\%} & \textbf{0.00} & 5.59 \\ \cline{2-6}
     & TC-$\lambda$ & 84.0\% & 52.6\% & 0.62 & 3.24 \\ \cline{2-6}
     & \textcolor{orange!90!black}{\raisebox{0.1ex}{\large$\bigstar$}} TCRM (ours) & 84.2\% & 63.5\% & 0.64 & \textbf{2.18} \\ \hline
    \multirow{4}{*}{Qwen3 1.7B} & Baseline & 86.5\% & 52.3\% & 3.91 & 6.22 \\ \cline{2-6}
     & ImplicitPRM (DPO) & 79.8\% & \textbf{75.3\%} & \textbf{0.00} & 6.98 \\ \cline{2-6}
     & TC-$\lambda$ & 86.5\% & 54.9\% & 0.64 & 3.65 \\ \cline{2-6}
     & \textcolor{orange!90!black}{\raisebox{0.1ex}{\large$\bigstar$}} TCRM (ours) & 86.5\% & 57.6\% & 0.67 & \textbf{3.19} \\ \hline
    \multirow{4}{*}{Qwen3 4B} & Baseline & 90.3\% & 50.4\% & 8.12 & 20.58 \\ \cline{2-6}
     & ImplicitPRM (DPO) & 83.3\% & \textbf{78.0\%} & \textbf{0.00} & 14.67 \\ \cline{2-6}
     & TC-$\lambda$ & 89.7\% & 64.5\% & 0.36 & 3.25 \\ \cline{2-6}
     & \textcolor{orange!90!black}{\raisebox{0.1ex}{\large$\bigstar$}} TCRM (ours) & 89.7\% & \textbf{78.1\%} & 0.26 & \textbf{1.68} \\ \hline
    \multirow{4}{*}{Qwen3 8B} & Baseline & 90.3\% & 52.0\% & 4.27 & 10.68 \\ \cline{2-6}
     & ImplicitPRM (DPO) & 83.0\% & \textbf{77.8\%} & \textbf{0.00} & 8.73 \\ \cline{2-6}
     & TC-$\lambda$ & 90.0\% & 62.0\% & 0.34 & 3.88 \\ \cline{2-6}
     & \textcolor{orange!90!black}{\raisebox{0.1ex}{\large$\bigstar$}} TCRM (ours) & 90.4\% & 74.3\% & 0.23 & \textbf{2.22} \\ \hline
    \multirow{4}{*}{Qwen3 14B} & Baseline & 93.2\% & 60.1\% & 1.79 & 5.16 \\ \cline{2-6}
     & ImplicitPRM (DPO) & 83.4\% & 79.7\% & \textbf{0.00} & 4.88 \\ \cline{2-6}
     & TC-$\lambda$ & 92.5\% & 86.3\% & 0.04 & 1.05 \\ \cline{2-6}
     & \textcolor{orange!90!black}{\raisebox{0.1ex}{\large$\bigstar$}} TCRM (ours) & 93.6\% & \textbf{88.7\%} & 0.02 & \textbf{0.85} \\ \hline
    \multirow{3}{*}{Qwen3 32B} & Baseline & 93.5\% & 60.1\% & 0.99 & 2.10 \\ \cline{2-6}
     & ImplicitPRM (DPO) & 82.6\% & 78.3\% & \textbf{0.00} & 4.78 \\ \cline{2-6}
     & \textcolor{orange!90!black}{\raisebox{0.1ex}{\large$\bigstar$}} TCRM (ours) & 93.6\% & \textbf{88.9\%} & 0.01 & \textbf{0.34} \\ \hline
     \multirow{4}{*}{Llama 3.1 8B} & Baseline & 92.6\% & 51.1\% & 9.18 & 13.62 \\ \cline{2-6}
     & ImplicitPRM (DPO) & 84.4\% & 79.8\% & \textbf{0.02} & 56.72 \\ \cline{2-6}
     & TC-$\lambda$ & 92.3\% & 71.7\% & 0.08 & 2.88 \\ \cline{2-6}
     & \textcolor{orange!90!black}{\raisebox{0.1ex}{\large$\bigstar$}} TCRM (ours) & 92.5\% & \textbf{84.7\%} & 0.06 & \textbf{1.22} \\ \hline
    \end{tabular}
\end{table}

Table \ref{tab:main_metrics} shows a clear ranking. TC-$\lambda$ and TCRM maintain essentially the same final-token accuracy as Baseline, while ImplicitPRM is consistently worse across all model sizes. ImplicitPRM has the best middle-token accuracy on smaller models but falls behind substantially on larger models, indicating weak scaling properties. TCRM is the most stable method for both types of accuracies across model sizes and families, and it consistently achieves the lowest Final Delta, which is the strongest evidence that its intermediate scores are the best approximation of the conditional expectation of the final reward.

Additionally, in Appendix \ref{sec:rewardbench_v2} we report the results from RewardBench 2 \cite{malik2025rewardbench} evaluation which show that TCRM consistently improves over all baselines at larger model scale, with especially strong gains in Safety and Focus---while competing token-level reward methods (ImplicitPRM, TC-$\lambda$) degrade on larger models.

In Appendix \ref{sec:cond_exp}, we test whether intermediate outputs are exact conditional expectations. Bias and error decrease for longer prefixes, but non-zero residual correlation remains, indicating that the conditional expectation property is only approximate.

\subsection{Interpretability}

We use the prompt "Give me the first 10 numbers of the Fibonacci sequence" and compare baseline vs TCRM token trajectories. In Figure \ref{fig:fib_wrong}, the response is "0,1,1,2,3,5,10,13,21,34" (wrong token: 10 instead of 8): the baseline trajectory is noisy, while TCRM sharply penalizes the error and then partially recovers as later tokens are correct. Additional response variants are shown in Appendix \ref{sec:extra_interpret}: the correct response "0,1,1,2,3,5,8,13,21,34" (Figure \ref{fig:fib_correct}) and an overlong response with an extra eleventh number (Figure \ref{fig:fib_extra}), where TCRM penalizes continuing beyond the requested 10 numbers.

\begin{figure}[ht] % Using [ht] for better placement
\centering
\begin{subfigure}{.5\textwidth}
  \centering
  \includegraphics[width=\linewidth]{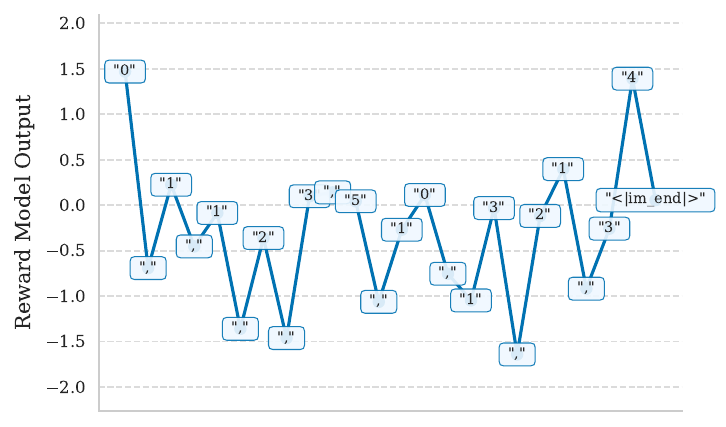} 
  \caption{Baseline reward model}
  \label{fig:fib_baseline_wrong}
\end{subfigure}% <-- The % is important to prevent extra space
\begin{subfigure}{.5\textwidth}
  \centering
  \includegraphics[width=\linewidth]{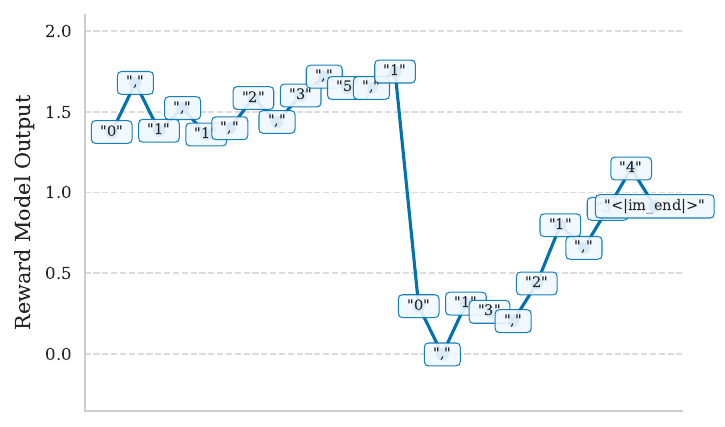}
  \caption{Temporally coherent reward model}
  \label{fig:fib_coherent_wrong}
\end{subfigure}
\caption{Token-level reward trajectories for a Fibonacci response with an incorrect token (10 instead of 8).}
\label{fig:fib_wrong}
\end{figure}

\section{Process Reward Model Evaluation}
\label{sec:process_bench}

Temporally coherent reward models can provide PRM-like step-level feedback while training only on outcome labels. We evaluate this on ProcessBench \cite{zheng2024processbench} against strong open-weight PRMs and outcome-only token-level baselines (ImplicitPRM and ABC).

Because ProcessBench is math-focused, we train a math-specialized TCRM on the same underlying reasoning dataset used by ImplicitPRM, enabling a like-for-like algorithmic comparison \footnote{The dataset has multi-step reasoning traces, but the labels are based only on the final answer correctness.}. Full setup, evaluation protocol, and threshold-selection details are shared in Appendix \ref{sec:pb_eval_details}.

\begin{table}[ht]
\centering
\caption{ProcessBench F1 (compact main-text view). Full leaderboard is in Appendix \ref{sec:pb_eval_details}.}
\label{tab:process_bench}
\begin{tabular}{lrrrrr}
\toprule
\textbf{Model} & \textbf{GSM8K} & \textbf{MATH} & \textbf{Olympiad-Bench} & \textbf{Omni-MATH} & \textbf{Average} \\
\midrule
Qwen2.5-Math-PRM-72B (step-labeled PRM) & \textbf{87.3} & \textbf{80.6} & \textbf{74.3} & \textbf{71.1} & \textbf{78.3} \\
Llama3.1-8B-ImplicitPRM-DPO & \textbf{72.1} & 46.0 & 28.0 & 26.7 & 43.2 \\
Llama3.1-8B-ImplicitPRM-CE & 67.6 & 46.2 & 27.5 & 29.7 & 42.8 \\
ABC-Llama3.1-8B & 38.5 & 29.3 & 16.0 & 15.1 & 24.7 \\
\textcolor{orange!90!black}{\raisebox{0.1ex}{\large$\bigstar$}} TCRM-Llama3.1-8B (trained on math data) & 68.9 & \textbf{47.7} & \textbf{34.8} & 28.3 & \textbf{44.9} \\
\bottomrule
\end{tabular}
\end{table}

Table \ref{tab:process_bench} shows a compact view of representative models. Among outcome-only methods, TCRM (math-trained) has the best average F1 (44.9), outperforming the strongest ImplicitPRM baseline by $+1.7$ points. Full step-labeled PRMs remain stronger overall, but they use substantially richer supervision.

\section{Reinforcement Learning}
\label{sec:rl}

The main application of reward models is as a source of reward signal in Reinforcement Learning training. In this section we evaluate the effectiveness of TCRM compared to a regular Bradley-Terry reward model (baseline). We use PPO since it is the most commonly used algorithm with learned reward models in RLHF. The reward models are trained on the Skywork-Reward-Preference-80K-v0.2 \cite{liu2024skywork} dataset and the prompts for RL are sourced from \textit{Dolci-Instruct-RL} dataset used to train Olmo 3 \cite{olmo2025olmo3}. More details and hyperparameters are reported in Appendix \ref{sec:rl_training}.

PPO learns a value model during training, and TCRM is explicitly designed to share this objective structure---the correspondence between our regularizers and standard Monte Carlo / TD value-learning objectives is established in Section~\ref{sec:value-learning-connection}. In practice, PPO often uses GAE, which interpolates between MC and TD targets; TCRM's combined loss plays an analogous role. Therefore, TCRM can be expected to function both as a reward model and as a value model. We evaluate 3 possible ways of using TCRM for value modeling \footnote{With all of these approaches we also use TCRM as a reward model to get the reward score from the final token.} in PPO:
\begin{enumerate}
    \item \textbf{Full fine-tuning}. TCRM is used to initialize the value model to enable a "warm start" since it already has properties very similar to the value model. After initialization, the standard value model loss is used to fine-tune the value model.
    \item \textbf{Frozen value model}. TCRM is used as the value model with no extra fine-tuning. This approach is most resource-efficient since it eliminates the need for having separate reward and value models (the same model serves as both reward and value model). It also eliminates value model optimization step, reducing both memory and compute requirements.
    \item \textbf{LoRA value model}. A frozen TCRM is used as the reward model and a LoRA adapter is trained on top of it for the value model. This approach significantly reduces memory requirements because it removes the need for a separate value model and significantly reduces the size of value model gradients and optimizer state.
\end{enumerate}

Figure \ref{fig:rlhf_progress} compares PPO training progress across value-model setups (see Appendix \ref{sec:rlhf_eval} for evaluation details). Our primary takeaway is efficiency: TCRM-based setups preserve policy quality while reducing training cost. In particular, the LoRA value-model variant is attractive because it substantially lowers memory requirements while maintaining comparable quality.

\begin{figure}
    \centering
    \includegraphics[width=0.75\linewidth]{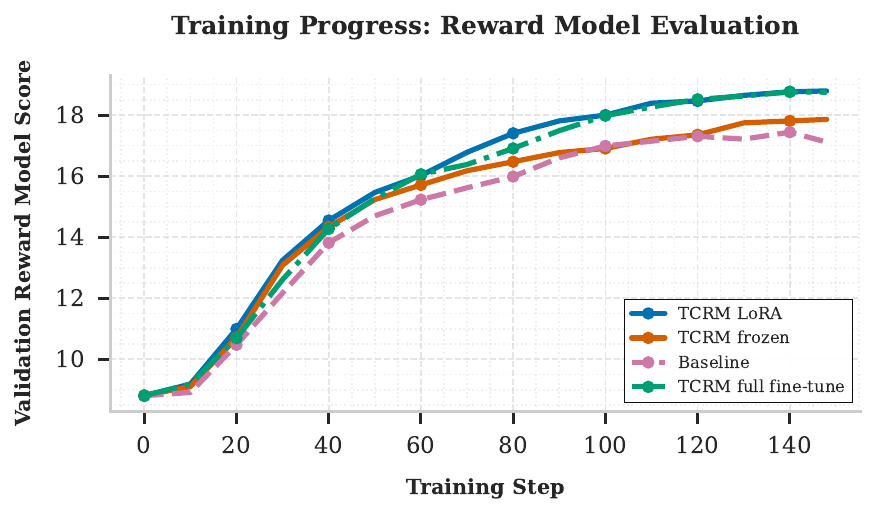}
    \caption{Validation reward model scores for PPO training run with different value model setups.}
    \label{fig:rlhf_progress}
\end{figure}

TCRM estimates the conditional expectation of reward under the assumption that future tokens follow the training distribution. During PPO, generation distributions shift as the policy changes; empirically, however, even a frozen TCRM maintains quality comparable to the more resource-intensive baseline, which supports using TCRM for efficiency gains.

Table \ref{tab:llm_judge} reports an LLM-as-a-Judge comparison between baseline and TCRM (LoRA) reward/value setups. The two LLMs have statistically equivalent quality (TCRM 33.7\% vs baseline 34.3\%; $p=0.85$ for a two-sided test), despite lower resource requirements of TCRM.

\begin{table}[ht]
    \centering
    % Set row height to be slightly more generous for readability
    \renewcommand{\arraystretch}{2}
    
    \caption{LLM-as-a-Judge comparison of LLMs trained with PPO using baseline and TCRM reward models.}
    \label{tab:llm_judge}
    \begin{tabular}{| @{} >{\centering\arraybackslash}p{0.17\textwidth} | 
                      @{} >{\centering\arraybackslash}p{0.17\textwidth} | 
                      @{} >{\centering\arraybackslash}p{0.16\textwidth} @{} |}
        \hline
        
        % First Column: TCRM (spans 2 rows)
        \multirow{2}{=}{\centering \textbf{\textcolor{orange!90!black}{\raisebox{0.1ex}{\large$\bigstar$}} TCRM better} \\ 33.7\%} & 
        
        % Second Column, Top Row: Tie
        Tie \quad 9.2\% & 
        
        % Third Column: Baseline (spans 2 rows)
        \multirow{2}{=}{\centering \textbf{Baseline better} \\ 34.3\%} \\
        
        % Middle horizontal line for the center column only
        \cline{2-2}
        
        % Second Row: Order-depends
        & Depends on order \quad 22.8\% & \\
        \hline
    \end{tabular}
\end{table}

One of the benefits of using TCRM during PPO training is that the value model already has good accuracy from the very first step, providing a "warm start" benefit. In contrast, a standard PPO value model is initialized from a non-coherent reward model and takes 20-30 steps to converge to good accuracy. The policy LLM can suffer from inaccurate reward signals before this convergence happens. In Appendix \ref{sec:rlhf_add_result} we show that freezing the policy LLM for 25 steps to let the value model converge significantly underperforms PPO with unfrozen policy LLM, so it is not a good alternative to a warm start with a TCRM-based value model.

In Appendix \ref{sec:rlhf_add_result} we further show that PPO with TCRM-based reward/value modeling is substantially more resource-efficient and uses KL budget more effectively. Specifically, using a frozen TCRM as both reward and value models reduces peak GPU memory by 27\%, training step time by 19\%, and KL divergence by 26\%, while maintaining comparable validation reward-model scores.

\section{Related Work}
\label{sec:related_work}

\subsection{Reinforcement Learning from Human Feedback}

Reinforcement Learning from Human Feedback (RLHF) was first introduced in \cite{christiano2017rlhf} as an alternative to rigid reward functions used in classical RL environments like MuJoCo (simulated robotics) and Atari (video games). It was later popularized by successful applications to LLMs in \cite{ouyang2022instructgpt} and \cite{bai2022anthropichh}, where it found widespread use. The core RLHF recipe includes training a human preference reward model with a pairwise Bradley-Terry loss \cite{bradley1952bradleyterry} and using this reward model as a proxy for human preference in an online RL training algorithm like PPO \cite{schulman2017ppo} or RLOO \cite{ahmadian2024rloo}. Most relevant to our paper, \cite{wu2023fine} show that using dense token-level feedback can significantly improve the effectiveness of RLHF. Similarly, \cite{yin2025segmenting} use segment-level rewards by assigning the rewards to multi-token segments.

\subsection{Fine-Grained Reward Models}

Since fine-grained reward can improve RLHF performance, recipes for training fine-grained reward models have been researched. These methods largely fall into three categories, depending on the training data labeling process.

First, human annotation can be used to label the data either at a step level or at a token level. Step-level annotation is most common for Process Reward Models (PRM) used for mathematical problem solving, pioneered by \cite{lightman2023letsverify}. Token-level human annotations were used by \cite{wu2023fine} to label the irrelevant or incorrect passages. These papers provided an early proof of concept for the use of dense rewards, but the labor-intensive human annotation methods can not scale to the data volumes used for modern AI training.

Second, a powerful LLM can be used to either edit the response  or critique it, identifying valuable or flawed spans of text, which are then converted into token-level rewards. \cite{guo2023figa,chen2024rlmec,yoon2024tlcr} develop similar methods to rewrite the candidate response using a more powerful LLM and then assign token quality labels based on whether the token was kept, removed or replaced. Similarly, \cite{cao2024relc} prompt a powerful critic LLM to identify the positive and negative text spans and assign the token-level rewards according to that. More recently, step-wise annotations have been auto-generated by an LLM to train a mathematical PRM \cite{zhang2025lessons}. This category of papers relies on costly and often proprietary external models. Moreover, the quality of labeling is upper-bounded by these models' abilities to perform surgical edits and exactly identify positive/negative passages.

Finally, we can use only the outcome labels to train a dense token-level reward model by relying on inductive biases from reward model architecture or losses. Similar to our approach, \cite{yang2023preference} apply supervision to the reward model outputs at all tokens, but their loss doesn't induce any temporal coherence among the outputs. \cite{chan2024abc} develop a method which relies on the attention weights of the transformer to assign the final reward prediction to individual tokens proportional to the attention value from the last token to all tokens of the sequence. This is an inference-only technique with no changes required during training. \cite{yuan2024free} and \cite{rafailov2024dpoqstar} develop reward modeling methods based on Direct Preference Optimization (DPO), utilizing the inductive bias of parameterizing the reward as $r(x,y)=\log \frac{\pi(y|x)}{\pi_{ref}(y|x)}$. \cite{maystre2025incremental} uses an approach similar to ours, but only applies it to a text classification problem. Our paper contributes to this outcome-labeled direction and improves the accuracy compared to all other models in this category. Our key differential contributions are the use of theoretically motivated regularization terms in the reward model training loss to induce the conditional expectation property, as well as the use of TCRM as the value model to substantially improve resource efficiency of PPO RLHF training.

\section{Conclusion}

We introduced Temporally Coherent Reward Modeling (TCRM), a reward-model training framework that supervises intermediate token outputs to approximate the expected final reward. TCRM augments standard Bradley--Terry training with temporal regularization terms chosen so that their minimizers are conditional expectations, and it does not require changes to model architecture or training data format.

Across model scales, TCRM substantially improves prediction quality on partial responses while preserving final-token performance. On ProcessBench, TCRM achieves strong PRM performance among outcome-only methods (44.9\% average F1). In PPO, using TCRM as a combined reward/value model improves practical efficiency: in our setup, a frozen TCRM reward/value model reduces peak GPU memory by 27\% and training step time by 19\% while maintaining comparable or better validation reward-model scores.

A key limitation of our results is that intermediate outputs are not exact conditional expectations: Appendix \ref{sec:cond_exp} shows decreasing bias/error with longer prefixes, but non-zero residual correlation remains. Future work includes improving conditional-expectation calibration, using entropy-aware selective smoothing, and leveraging dense token-level rewards for online RL and fine-grained error correction (e.g., safety violations and reasoning defects).

\section{Impact Statement}

This paper presents work whose goal is to advance the field
of Machine Learning. There are many potential societal
consequences of our work, none which we feel must be
specifically highlighted here.

\clearpage
\newpage
\bibliographystyle{assets/plainnat}
\bibliography{paper}

\clearpage
\newpage
\beginappendix

\section{Conditional Expectation Lemmas}

\begin{lemma}
\label{lem:lookahead_cond_exp}
Let $(\Omega, \mathcal{F}, P)$ be a probability space equipped with a filtration $(\mathcal{F}_t)_{t=0}^\infty$, and let $T$ be a stopping time with respect to $(\mathcal{F}_t)$ satisfying $P(T > k) > 0$ for a fixed $k \in \mathbb{N}$. Let $Z_T = h(X_0, \ldots, X_T) \in L^2(\Omega, \mathcal{F}_T, P)$ be a given terminal value, where $(X_t)_{t=0}^T$ is an $(\mathcal{F}_t)$-adapted process. Consider the problem
\[
  f_k^* \;=\; \underset{f \,\in\, L^2(\mathcal{F}_k)}{\operatorname{arg\,min}}\; E\!\Big[\big(Z_T - f\big)^2 \;\Big|\; T > k\Big].
\]
Then the minimizer exists, is unique $($a.s.$)$, and is given by
\[
  f_k^* \;=\; E\big[Z_T \mid \mathcal{F}_k,\; T > k\big].
\]
\end{lemma}

\begin{proof}[Proof of Lemma ~\ref{lem:lookahead_cond_exp}]
Since $T$ is a stopping time, $\{T > k\} \in \mathcal{F}_k$, so conditioning on $(\mathcal{F}_k,\, T > k)$ coincides with conditioning on $\mathcal{F}_k$ on the event $\{T > k\}$. Set $f_k^* := E[Z_T \mid \mathcal{F}_k]$, which lies in $L^2(\mathcal{F}_k)$ by Jensen's inequality. For any $f \in L^2(\mathcal{F}_k)$, write
\[
  E\!\Big[\big(Z_T - f\big)^2 \;\Big|\; T > k\Big]
  \;=\; E\!\Big[\big(Z_T - f_k^*\big)^2 \;\Big|\; T > k\Big]
  \;+\; 2\,E\!\Big[\big(Z_T - f_k^*\big)\big(f_k^* - f\big) \;\Big|\; T > k\Big]
  \;+\; E\!\Big[\big(f_k^* - f\big)^2 \;\Big|\; T > k\Big].
\]
The cross term vanishes: $(f_k^* - f)\,\mathbf{1}_{\{T>k\}}$ is $\mathcal{F}_k$-measurable, so the tower property gives
\[
  E\!\Big[\big(Z_T - f_k^*\big)\big(f_k^* - f\big)\,\mathbf{1}_{\{T>k\}}\Big]
  = E\!\Big[\underbrace{E\big[Z_T - f_k^* \mid \mathcal{F}_k\big]}_{=\,0}\;\big(f_k^* - f\big)\,\mathbf{1}_{\{T>k\}}\Big]
  = 0.
\]
Hence $E\big[(Z_T - f)^2 \mid T > k\big] \geq E\big[(Z_T - f_k^*)^2 \mid T > k\big]$, with equality if and only if $f = f_k^*$ a.s.\ on $\{T > k\}$.
\end{proof}

\begin{lemma}
\label{lem:smooth_cond_exp}
Let $(\Omega, \mathcal{F}, P)$ be a probability space equipped with a filtration $(\mathcal{F}_t)_{t=0}^T$, where $T < \infty$ is a bounded stopping time. Let $X_T \in L^2(\Omega, \mathcal{F}_T, P)$ be a given terminal value. Define the backward recursion: set $X_T^* := X_T$, and for $t = T{-}1, T{-}2, \ldots, 0$, let
\[
  X_t^* \;=\; \underset{Y \in L^2(\mathcal{F}_t)}{\operatorname{arg\,min}}\; E\!\Big[\big(X_{t+1}^* - Y\big)^2 \,\Big|\, \mathcal{F}_t\Big].
\]
Then the following hold:
\begin{enumerate}
  \item[\textup{(i)}] The minimizer exists, is unique $($a.s.$)$, and is given by
  \[
    X_t^* = E\big[X_{t+1}^* \mid \mathcal{F}_t\big], \qquad t = T{-}1, \ldots, 0.
  \]
  \item[\textup{(ii)}] The process $(X_t^*)_{t=0}^T$ is a martingale with respect to $(\mathcal{F}_t)_{t=0}^T$.
  \item[\textup{(iii)}] For every $t \in \{0, \ldots, T\}$,
  \[
    X_t^* = E\big[X_T \mid \mathcal{F}_t\big].
  \]
\end{enumerate}
\end{lemma}

\begin{proof}[Proof of Lemma ~\ref{lem:smooth_cond_exp}]
\textbf{(i)} Fix $t$ and suppose $X_{t+1}^* \in L^2$ is already defined. For any $\mathcal{F}_t$-measurable $Y \in L^2$, decompose
\[
  X_{t+1}^* - Y \;=\; \underbrace{\big(X_{t+1}^* - E[X_{t+1}^*|\mathcal{F}_t]\big)}_{\displaystyle =:\, \varepsilon_{t+1}} \;+\; \underbrace{\big(E[X_{t+1}^*|\mathcal{F}_t] - Y\big)}_{\displaystyle =:\, \delta_t}.
\]
The term $\varepsilon_{t+1}$ satisfies $E[\varepsilon_{t+1} \mid \mathcal{F}_t] = 0$, and $\delta_t$ is $\mathcal{F}_t$-measurable. Therefore, by the orthogonality of $L^2$-projections,
\begin{align*}
  E\!\big[(X_{t+1}^* - Y)^2 \mid \mathcal{F}_t\big]
  &= E\!\big[\varepsilon_{t+1}^2 \mid \mathcal{F}_t\big] + 2\,\delta_t \cdot \underbrace{E[\varepsilon_{t+1} \mid \mathcal{F}_t]}_{=\,0} + \delta_t^2 \\
  &= E\!\big[\varepsilon_{t+1}^2 \mid \mathcal{F}_t\big] + \delta_t^2.
\end{align*}
The first term does not depend on $Y$. The second term is non-negative and vanishes if and only if $Y = E[X_{t+1}^* \mid \mathcal{F}_t]$ almost surely. Since $X_{t+1}^* \in L^2$ implies $E[X_{t+1}^* \mid \mathcal{F}_t] \in L^2$ (by Jensen's inequality), the minimizer exists, is unique a.s., and equals $E[X_{t+1}^* \mid \mathcal{F}_t]$.
 
\medskip
\textbf{(ii)} By construction, $E[X_{t+1}^* \mid \mathcal{F}_t] = X_t^*$ for all $t$, which is the martingale property. Adaptedness holds since each $X_t^*$ is $\mathcal{F}_t$-measurable by definition, and $L^2$-integrability follows inductively from part~(i).
 
\medskip
\textbf{(iii)} We proceed by backward induction.
 
\emph{Base case.} $X_T^* = X_T = E[X_T \mid \mathcal{F}_T]$.
 
\emph{Inductive step.} Suppose $X_{t+1}^* = E[X_T \mid \mathcal{F}_{t+1}]$. Then
\[
  X_t^* = E\big[X_{t+1}^* \mid \mathcal{F}_t\big]
        = E\!\Big[E\big[X_T \mid \mathcal{F}_{t+1}\big] \,\Big|\, \mathcal{F}_t\Big]
        = E\big[X_T \mid \mathcal{F}_t\big],
\]
where the last equality follows from the tower property of conditional expectation.
\end{proof}

\section{Additional Training Details}
\label{section:add_train_det}
All reward models were trained on a single node with 8 H100 GPUs.

The open source TRL framework was used for training.

The batch size and parallelism strategy were chosen for each model size to maximize compute efficiency. These hyperparameters were always kept constant in reward models of the same size to maintain comparability.

\begin{table}[h]
    \centering
    \begin{tabular}{|c|c|}
        \hline
        \textbf{Hyperparameter} & \textbf{Value} \\ \hline
        Maximum response length & 1024 tokens \\ \hline
        Number of training epochs & 1 \\ \hline
        Learning rate & 1e-5 \\ \hline
        Optimizer & AdamW \\ \hline
        Batch size & 32 (all models 8B and smaller); 16 (Qwen 3 14B); 2 (Qwen 3 32B) \\ \hline
        Gradient accumulation steps & 1 (all models except Qwen 3 32B); 8 (Qwen 3 32B) \\ \hline
        Smoothness weight in the loss & 0.1 \\ \hline
        Lookahead consistency weight in the loss & 0.01 \\ \hline
    \end{tabular}
    \caption{Training hyperparameters}
    \label{tab:hyperparams}
\end{table}

As additional baselines, we also trained reward models using ImplicitPRM \cite{yuan2024free} and TC-$\lambda$ \cite{maystre2025incremental} methods. For ImplictPRM we used the DPO version since it had better results.
ImplicitPRM was trained with OpenRLHF since DPO training in TRL was unstable. Hypermarameters we $\beta=0.05$, learning rate $5 \times 10^{-7}$.
For TC-$\lambda$ the base training objective is the standard Bradley-Terry pairwise preference loss on the final token. We augment this with a temporal consistency (TC) regularization loss weighted by $\alpha = 0.1$.
We considered two TC loss variants: MSE in the logit space and BCE applied to sigmoids of logits. The MSE version had better accuracy, so we report results from that version.
We set $\lambda = 0.95$ and detach (stop-gradient, $\mathrm{sg}$) the targets to prevent backpropagation through the exponential smoothing chain. 

\section{Ablations}
\label{sec:ablations}

In Section \ref{section:temp_coh_reg} we present several different regularization loss terms which induce different desired properties of a temporally coherent reward model. Here we analyze the effect of each of these loss terms. All ablation experiments were based on Llama 3.1 8B.

We have 3 possible regularization loss terms and 3 temporal coherence metrics (in addition to final-token accuracy). Not surprisingly, each metric is best optimized by using its corresponding loss term. For all regularization types excessive regularization strength leads to substantial reduction in final-token accuracy - the outcome which we want to avoid. Because of this, for our main models reported in Table \ref{tab:main_metrics} we mix all 3 regularization terms with small coefficients (see Table \ref{tab:hyperparams} for exact configuration).

Comparing the metrics in tables \ref{tab:abl_lookahead_detach} and \ref{tab:abl_lookahead_no_detach} we see that not detaching the targets in lookahead consistency regularization term\footnote{The same is true for the smoothness loss term, table not included in the paper for brevity.} leads to faster degradation of the final-token accuracy. Because of this, in this paper we use the regularization terms (both smoothness and lookahead consistency) with detached targets.

\begin{table}[h]
    \centering
    \begin{tabular}{|c|c|c|c|c|}
    \hline \multirow{2}{*}{\textbf{Smoothness loss coefficient}} & \multicolumn{2}{c|}{\textbf{Accuracy}} & \multicolumn{2}{c|}{\textbf{Mean Squared Delta}} \\ \cline{2-5} 
     & \textbf{Final Token} & \textbf{Middle Token} & \textbf{Step} & \textbf{Final} \\ \hline
    0    & 92.6\% & 53.4\% & 7.98 & 11.41 \\ \hline
    1e-3 & 92.5\% &	55.4\% & 4.97 & 10.03 \\ \hline
    3e-3 & 92.6\% &	55.0\% & 3.05 & 7.90 \\ \hline
    1e-2 & 92.4\% &	55.3\% & 0.72 & 5.53 \\ \hline
    3e-2 & 92.4\% &	56.1\% & 0.15 & 4.77 \\ \hline
    1e-1 & 92.4\% &	61.1\% & 0.05 & 4.60 \\ \hline
    3e-1 & 92.3\% &	67.1\% & 0.04 & 4.68 \\ \hline
    1.0  & 91.9\% &	75.5\% & 0.03 & 3.45 \\ \hline
    3.0  & 91.5\% &	82.1\% & 0.01 & 1.80 \\ \hline
    10.0 & 87.3\% &	76.7\% & 0.01 & 1.09 \\ \hline
    \end{tabular}
    \caption{Metrics for temporally coherent models trained with smoothness regularization only.}
    \label{tab:abl_smooth}
\end{table}

\begin{table}[h]
    \centering
    \begin{tabular}{|c|c|c|c|c|}
    \hline \multirow{2}{*}{\textbf{Lookahead consistency loss coefficient}} & \multicolumn{2}{c|}{\textbf{Accuracy}} & \multicolumn{2}{c|}{\textbf{Mean Squared Delta}} \\ \cline{2-5} 
     & \textbf{Final Token} & \textbf{Middle Token} & \textbf{Step} & \textbf{Final} \\ \hline
    0    & 92.6\% & 53.4\% & 7.98 & 11.41 \\ \hline
    3e-3 & 92.4\% & 65.6\% & 2.55 & 4.45 \\ \hline
    1e-2 & 92.7\% & 77.9\% & 0.65 & 2.13 \\ \hline
    3e-2 & 92.0\% & 80.3\% & 0.50 & 1.73 \\ \hline
    1e-1 & 88.3\% & 66.8\% & 0.90 & 1.74 \\ \hline
    3e-1 & 88.4\% & 73.5\% & 0.21 & 1.04 \\ \hline
    1.0 & 72.8\% & 68.4\% & 0.05 & 0.23 \\ \hline
    3.0 & 79.7\% & 65.4\% & 0.11 & 0.48 \\ \hline
    \end{tabular}
    \caption{Metrics for temporally coherent models trained with lookahead consistency regularization only.}
    \label{tab:abl_lookahead}
\end{table}

\begin{table}[h]
    \centering
    \begin{tabular}{|c|c|c|c|c|}
    \hline \multirow{2}{*}{\textbf{Lookahead consistency loss coefficient}} & \multicolumn{2}{c|}{\textbf{Accuracy}} & \multicolumn{2}{c|}{\textbf{Mean Squared Delta}} \\ \cline{2-5} 
     & \textbf{Final Token} & \textbf{Middle Token} & \textbf{Step} & \textbf{Final} \\ \hline
    0    & 92.6\% & 53.4\% & 7.98 & 11.41 \\ \hline
    3e-3 & 92.4\% & 65.6\% & 2.55 & 4.45 \\ \hline
    1e-2 & 92.7\% & 77.9\% & 0.65 & 2.13 \\ \hline
    3e-2 & 92.0\% & 80.3\% & 0.50 & 1.73 \\ \hline
    1e-1 & 88.3\% & 66.8\% & 0.90 & 1.74 \\ \hline
    3e-1 & 88.4\% & 73.5\% & 0.21 & 1.04 \\ \hline
    1.0 & 72.8\% & 68.4\% & 0.05 & 0.23 \\ \hline
    3.0 & 79.7\% & 65.4\% & 0.11 & 0.48 \\ \hline
    \end{tabular}
    \caption{Metrics for temporally coherent models trained with lookahead consistency regularization only (targets detached).}
    \label{tab:abl_lookahead_detach}
\end{table}

\begin{table}[h]
    \centering
    \begin{tabular}{|c|c|c|c|c|}
    \hline \multirow{2}{*}{\textbf{Lookahead consistency loss coefficient}} & \multicolumn{2}{c|}{\textbf{Accuracy}} & \multicolumn{2}{c|}{\textbf{Mean Squared Delta}} \\ \cline{2-5} 
     & \textbf{Final Token} & \textbf{Middle Token} & \textbf{Step} & \textbf{Final} \\ \hline
    0    & 92.6\% & 53.4\% & 7.98 & 11.41 \\ \hline
    3e-3 & 92.8\% & 65.9\% & 2.17 & 3.80 \\ \hline
    1e-2 & 92.6\% & 77.8\% & 0.64 & 1.69 \\ \hline
    3e-2 & 92.3\% & 82.0\% & 0.22 & 0.92 \\ \hline
    1e-1 & 90.8\% & 81.3\% & 0.08 & 0.45 \\ \hline
    3e-1 & 89.4\% & 72.8\% & 0.03 & 0.18 \\ \hline
    1.0 & 83.8\% & 70.5\% & 0.01 & 0.04 \\ \hline
    3.0 & 83.6\% & 81.0\% & 0.002 & 0.01 \\ \hline
    \end{tabular}
    \caption{Metrics for temporally coherent models trained with lookahead consistency regularization only (targets not detached).}
    \label{tab:abl_lookahead_no_detach}
\end{table}

\section{Cross-Dataset Generalization}
\label{sec:cross_dataset}

To better understand the generalization properties of TCRM, we use 3 different dataset for reward model training and evaluation. We train a TCRM on each of the datasets and compare it to a standard Bradley-Terry reward model (baseline). The following tables show the comparison of final-token accuracy (\ref{tab:cross_dataset_final_token_accuracy}), middle-token accuracy (\ref{tab:cross_dataset_middle_token_accuracy}) and MSE of intermediate token output compared to the final token output (\ref{tab:cross_dataset_final_token_MSE}) between TCRM and baseline models.

We used the following datasets:
1. Helpful-Harmless \cite{bai2022anthropichh}
2. Skywork-Reward-Preference-80K-v0.2 \cite{liu2024skywork}
3. UltraFeedback \cite{cui2023ultrafeedback}

In most cases we see that for each evaluation dataset the model trained on the same \footnote{Each dataset was split into train/test subsets, so there was no direct data leakage} dataset performs the best. For final-token accuracy we see that a TCRM model generalizes as well or better than the baseline reward model, especially for the UltraFeedback evaluation dataset. For middle-token accuracy the baseline model has close to random accuracy, as should be expected. The TCRM has much better middle-token accuracies and they are also the highest when training and evaluating on the same dataset. The reductions in middle-token accuracy for off-diagonal table entries are similar to those observed for final-token accuracy, highlighting that the quality of intermediate token predictions in TCRM generalizes as much as the quality of final token predictions. The MSE to the final token is somewhat sensitive to data distribution shifts, but there are also surprising results like the model trained on UltraFeedback having the lowest MSE for evaluation on the Skywork dataset.

\begin{table}[ht]
    \centering
    \caption{Final-token accuracy comparison of baseline and TCRM models}
    \label{tab:cross_dataset_final_token_accuracy}

    % --- First Subtable (Left) ---
    \begin{subtable}[t]{0.45\textwidth}
        \centering
        \caption{Baseline}
        % \label{tab:experiment_a}
        \begin{tabular}{l c c c}
            \toprule
             & \multicolumn{3}{c}{\textbf{Training dataset}} \\
            \cmidrule(lr){2-4}
            \textbf{Evaluation dataset} & \textbf{HH} & \textbf{Skywork} & \textbf{UltraFB} \\
            \midrule
            HH & \textbf{72.2\%} & 61.9\% & 58.1\% \\
            Skywork & 74.0\% & \textbf{92.6\%} & 82.6\% \\
            UltraFB & 71.1\% & 71.7\% & \textbf{81.1\%} \\
            \bottomrule
        \end{tabular}
    \end{subtable}
    \hfill
    % --- Second Subtable (Right) ---
    \begin{subtable}[t]{0.45\textwidth}
        \centering
        \caption{\textcolor{orange!90!black}{\raisebox{0.1ex}{\large$\bigstar$}} TCRM}
        \begin{tabular}{l c c c}
            \toprule
             & \multicolumn{3}{c}{\textbf{Training dataset}} \\
            \cmidrule(lr){2-4}
            \textbf{Evaluation dataset} & \textbf{HH} & \textbf{Skywork} & \textbf{UltraFB} \\
            \midrule
            HH &\textbf{ 72.5\%} & 62.6\% & 56.5\% \\
            Skywork & 73.4\% & \textbf{92.1\%} & 83.1\% \\
            UltraFB & 72.4\% & 75.4\% & \textbf{81.9\%} \\
            \bottomrule
        \end{tabular}
    \end{subtable}
\end{table}

\begin{table}[ht]
    \centering
    \caption{Middle-token accuracy comparison of baseline and TCRM models}
    \label{tab:cross_dataset_middle_token_accuracy}

    % --- First Subtable (Left) ---
    \begin{subtable}[t]{0.45\textwidth}
        \centering
        \caption{Baseline}
        % \label{tab:experiment_a}
        \begin{tabular}{l c c c}
            \toprule
             & \multicolumn{3}{c}{\textbf{Training dataset}} \\
            \cmidrule(lr){2-4}
            \textbf{Evaluation dataset} & \textbf{HH} & \textbf{Skywork} & \textbf{UltraFB} \\
            \midrule
            HH & \textbf{53.3\%} & 50.0\% & 50.7\% \\
            Skywork & \textbf{53.8\%} & 53.4\% & 48.9\% \\
            UltraFB & \textbf{53.2\%} & 50.7\% & 52.6\% \\
            \bottomrule
        \end{tabular}
    \end{subtable}
    \hfill
    % --- Second Subtable (Right) ---
    \begin{subtable}[t]{0.45\textwidth}
        \centering
        \caption{\textcolor{orange!90!black}{\raisebox{0.1ex}{\large$\bigstar$}} TCRM}
        \begin{tabular}{l c c c}
            \toprule
             & \multicolumn{3}{c}{\textbf{Training dataset}} \\
            \cmidrule(lr){2-4}
            \textbf{Evaluation dataset} & \textbf{HH} & \textbf{Skywork} & \textbf{UltraFB} \\
            \midrule
            HH & \textbf{66.1\%} & 58.0\% & 54.5\% \\
            Skywork & 66.5\% & \textbf{84.6\%} & 71.4\% \\
            UltraFB & 66.7\% & 71.1\% & \textbf{76.1\%} \\
            \bottomrule
        \end{tabular}
    \end{subtable}
\end{table}

\begin{table}[ht]
    \centering
    \caption{MSE (of intermediate reward scores compared to the final reward score) comparison of baseline and TCRM models}
    \label{tab:cross_dataset_final_token_MSE}

    % --- First Subtable (Left) ---
    \begin{subtable}[t]{0.45\textwidth}
        \centering
        \caption{Baseline}
        % \label{tab:experiment_a}
        \begin{tabular}{l c c c}
            \toprule
             & \multicolumn{3}{c}{\textbf{Training dataset}} \\
            \cmidrule(lr){2-4}
            \textbf{Evaluation dataset} & \textbf{HH} & \textbf{Skywork} & \textbf{UltraFB} \\
            \midrule
            HH & \textbf{4.63} & 7.64 & 7.04 \\
            Skywork & \textbf{7.80} & 10.96 & 10.07 \\
            UltraFB & \textbf{7.14} & 8.66 & 8.59 \\
            \bottomrule
        \end{tabular}
    \end{subtable}
    \hfill
    % --- Second Subtable (Right) ---
    \begin{subtable}[t]{0.45\textwidth}
        \centering
        \caption{\textcolor{orange!90!black}{\raisebox{0.1ex}{\large$\bigstar$}} TCRM}
        \begin{tabular}{l c c c}
            \toprule
             & \multicolumn{3}{c}{\textbf{Training dataset}} \\
            \cmidrule(lr){2-4}
            \textbf{Evaluation dataset} & \textbf{HH} & \textbf{Skywork} & \textbf{UltraFB} \\
            \midrule
            HH & \textbf{0.39} & 0.94 & 1.05 \\
            Skywork & 0.52 & 1.23 & \textbf{0.45} \\
            UltraFB & 0.50 & 1.28 & \textbf{0.48} \\
            \bottomrule
        \end{tabular}
    \end{subtable}
\end{table}

\section{Conditional Expectation Evaluation}
\label{sec:cond_exp}

In this section we check to what extent the conditional expectation property $r(x,y_{0..k})=E[r(x,y)|x,y_{0..k}]$ holds, where $r(x,y)$ is the reward model output at the final token. We'll approach this analysis in 3 ways:

\begin{enumerate}
    \item Bias. If $r(x,y_{0..k})=E[r(x,y)|x,y_{0..k}]$, then the average value of $r(x,y_{0..k}) - r(x,y)$ should be close to 0.
    \item Decreasing prediction error. Since $y_{0..k}$ is more informative than $y_{0..k-1}$, we should expect to see the mean squared prediction error $(r(x,y_{0..k}) - r(x,y))^2$ decrease as $k$ increases.
    \item Orthogonality principle. If $f(x)=E[Y|X=x]$, then the residual $Y-f(x)$ must be uncorrelated with any function of $X$. We will use this principle in a simple form by checking the correlation of the residual with the intermediate score itself $Corr(r(x,y) - r(x,y_{0..k}), r(x,y_{0..k}))$.
\end{enumerate}

All metrics will be broken down by relative position within the response (100 bucket representing percentiles of position within response). Also, the metrics will be broken down by whether they are calculated on chosen or rejected response from the Skywork dataset. Figure \ref{fig:cond_exp_test} shows the results. The left plot depicts how the bias changes in a very intuitive pattern as the position increases - for chosen responses we start by underestimating their quality and then gradually converge towards unbiased estimate, while for rejected responses we start by overestimating them and then the positive bias gradually decreases as we observe more of the response. On the central plot we see that the prediction error decreases smoothly as we move towards the end of the response. A few interesting patterns in these curves: (1) The error does not converge to zero towards the end of the response because the EOS token often has a large jump in reward score - this is expected because until the EOS token is generated it's still possible that there could be a continuation where wrong/correct information is added; (2) for rejected responses (red curve) there is initially a rapid reduction in error, followed by slower reduction. The initial rapid reduction corresponds to identifying wrong information early in the response, while slower decline later on represents lack of correct information in the response (when we can't immediately conclude that the response  will be bad and need to wait to confirm that the correct response doesn't appear later). The correlation test (right plot) shows that intermediate model outputs aren't strictly conditional expectations of the final reward score. Instead, the residual is positively correlated with the current score, signaling that the model isn't using the information in the input to the full extent.

\begin{figure}
    \centering
    \includegraphics[width=1\linewidth]{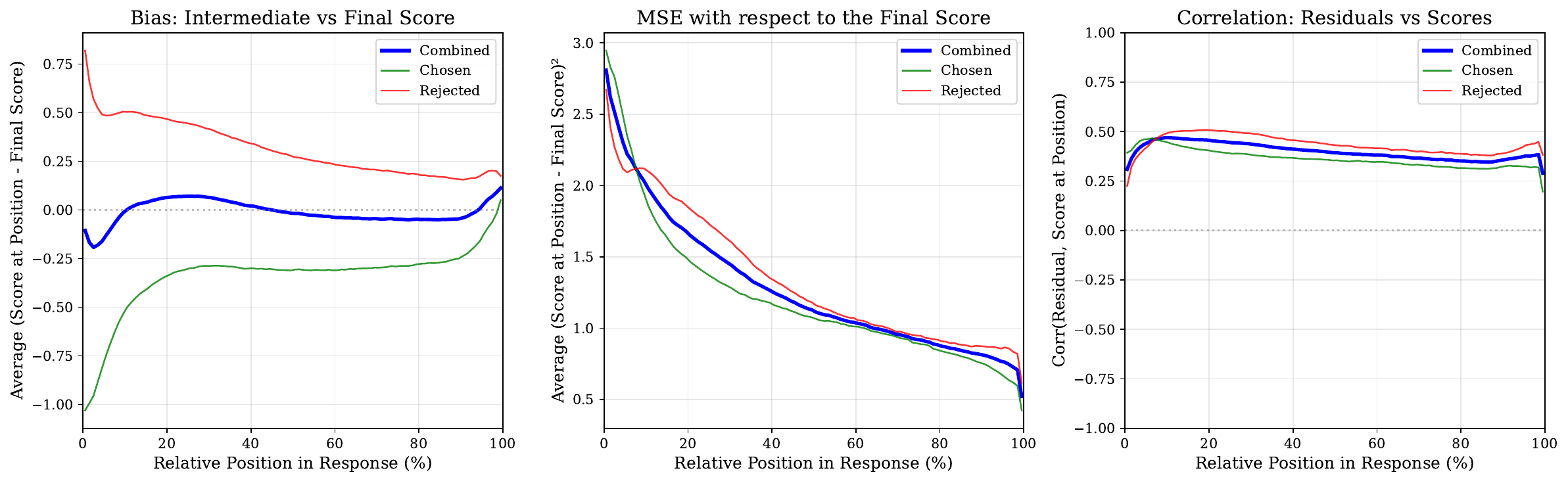}
    \caption{Conditional expectation tests. (Left) Bias; (Center) Prediction error; (Right) Correlation of residual with intermediate score.}
    \label{fig:cond_exp_test}
\end{figure}

\section{Additional Interpretability Results}
\label{sec:extra_interpret}

\begin{figure}[ht]
\centering
\begin{subfigure}{.5\textwidth}
  \centering
  \includegraphics[width=\linewidth]{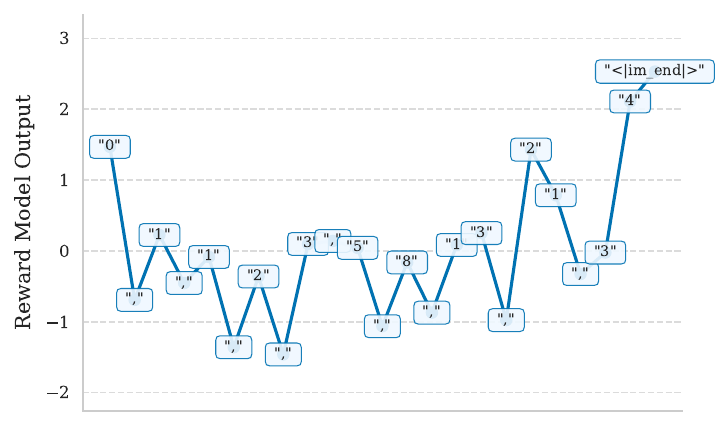}
  \caption{Baseline reward model}
  \label{fig:fib_baseline_correct}
\end{subfigure}
\begin{subfigure}{.5\textwidth}
  \centering
  \includegraphics[width=\linewidth]{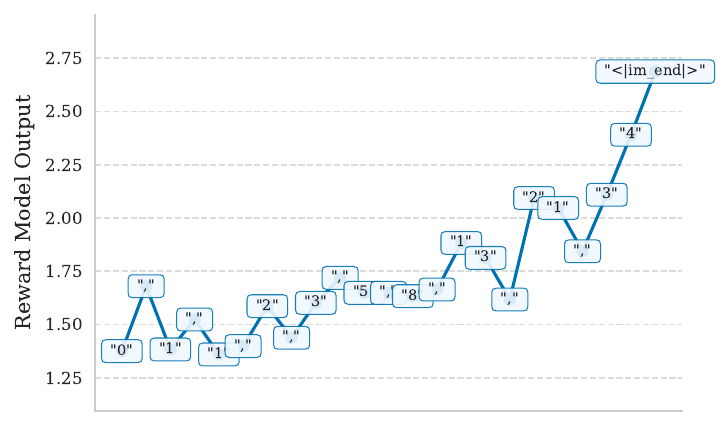}
  \caption{Temporally coherent reward model}
  \label{fig:fib_coherent_correct}
\end{subfigure}
\caption{Token-level reward trajectories for a correct Fibonacci response (10 numbers).}
\label{fig:fib_correct}
\end{figure}

\begin{figure}[ht] % Using [ht] for better placement
\centering
\begin{subfigure}{.5\textwidth}
  \centering
  \includegraphics[width=\linewidth]{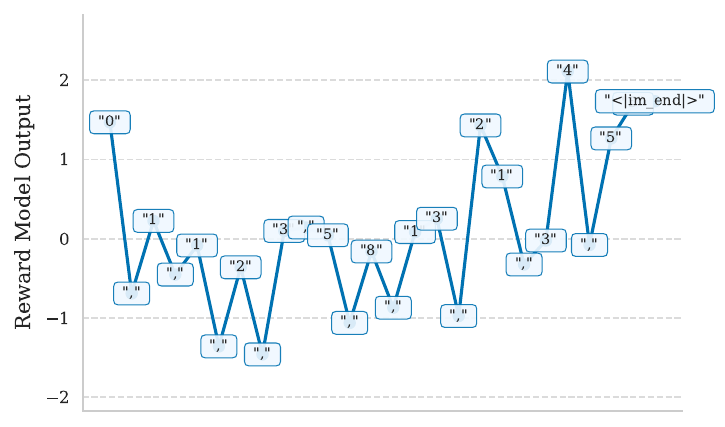} 
  \caption{Baseline reward model}
  \label{fig:fib_baseline_extra}
\end{subfigure}% <-- The % is important to prevent extra space
\begin{subfigure}{.5\textwidth}
  \centering
  \includegraphics[width=\linewidth]{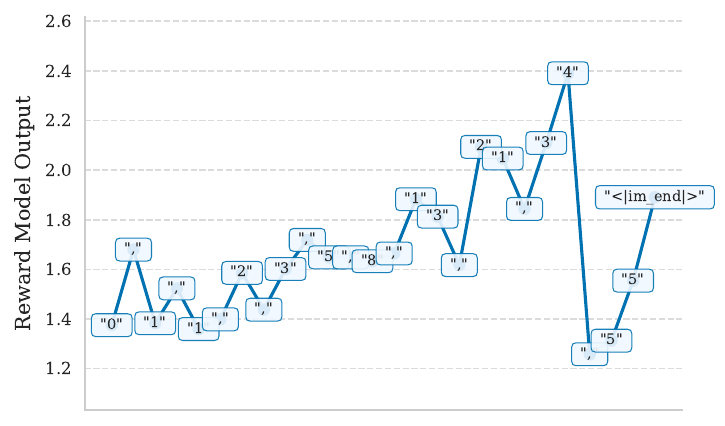}
  \caption{Temporally coherent reward model}
  \label{fig:fib_coherent_extra}
\end{subfigure}
\caption{Token-level reward trajectories for an overlong Fibonacci response (11 numbers instead of 10).}
\label{fig:fib_extra}
\end{figure}

\subsection{General Question Answering}

We use the prompt "What is the capital of Italy?" and the responses are scored by a temporally coherent reward model based on Qwen3-32B. The following candidate responses are considered:
\begin{enumerate}
    \item Correct. "Rome"
    \item Correct, with extra prefix. "It is Rome"
    \item Correct, with extra suffix. "Rome, also known as The Eternal City"
    \item Wrong suffix 1. "Rome, GA"
    \item Wrong suffix 2. "Rome, also known as The City of Love"
\end{enumerate}

Figure \ref{fig:rome} shows the results, which generally follow very intuitive patterns. The most succinct correct answer gets assigned the highest final reward. Adding an extra uninformative prefix to the correct answer in Figure \ref{fig:rome_coherent_ok_prefix} first reduces the reward score (showing that direct answers without filler words are preferred by the reward model), but the reward score increases sharply after the correct answer appears in the response. Appending a factually correct suffix in Figure \ref{fig:rome_coherent_ok_suffix} first reduces the reward score, but then it partially recovers after the reward model confirms that there are no factual errors in the suffix. Appending wrong suffixes to the correct answer in Figures \ref{fig:rome_coherent_wrong_suffix1} and \ref{fig:rome_coherent_wrong_suffix2} reduces the reward score.

\begin{figure}[ht]
    \centering % Center all subfigures
    
    % --- Top Row ---
    \begin{subfigure}[b]{0.48\textwidth}
        \centering
        \includegraphics[width=\linewidth]{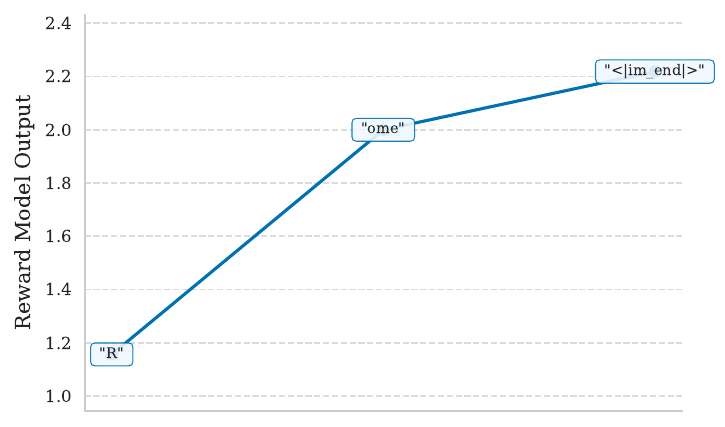}
        \caption{Correct response}
    \end{subfigure}
    \hfill % Adds horizontal space between the two images
    \begin{subfigure}[b]{0.48\textwidth}
        \centering
        \includegraphics[width=\linewidth]{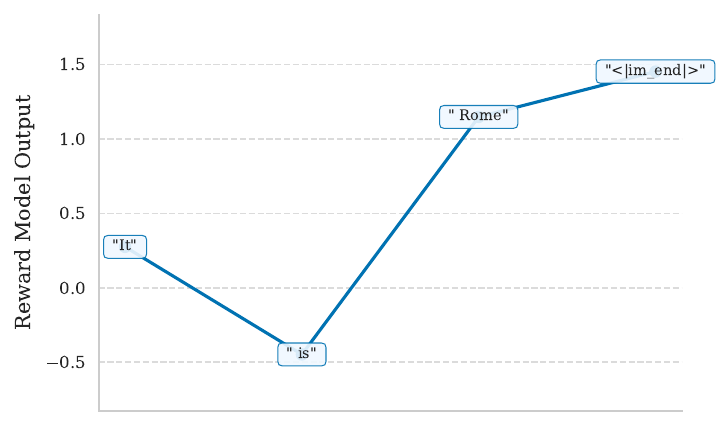}
        \caption{Correct response with extra prefix}
        \label{fig:rome_coherent_ok_prefix}
    \end{subfigure}
    
    % This blank line creates the new row
    
    % --- Middle Row ---
    \begin{subfigure}[b]{0.48\textwidth}
        \centering
        \includegraphics[width=\linewidth]{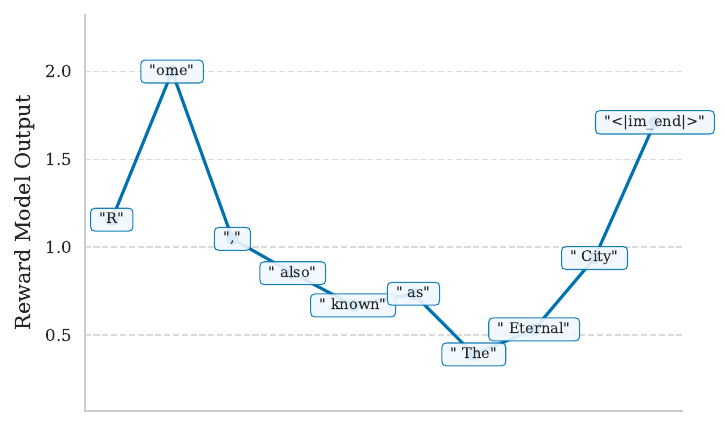} % Placeholder for your third image
        \caption{Correct response with extra suffix}
        \label{fig:rome_coherent_ok_suffix}
    \end{subfigure}
    \hfill % Adds horizontal space
    \begin{subfigure}[b]{0.48\textwidth}
        \centering
        \includegraphics[width=\linewidth]{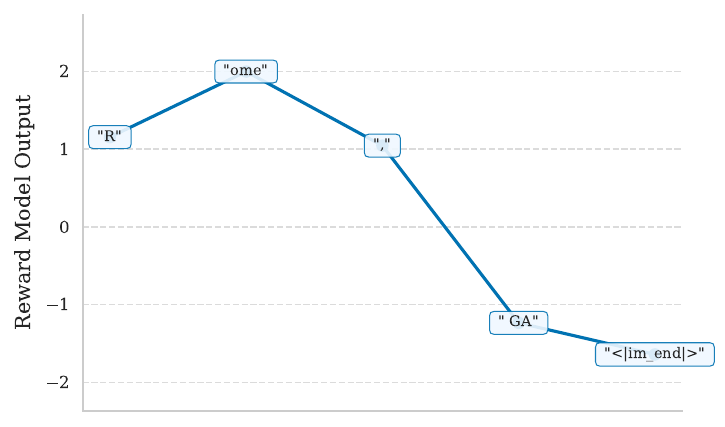} % Placeholder for your fourth image
        \caption{Wrong suffix 1}
        \label{fig:rome_coherent_wrong_suffix1}
    \end{subfigure}

    % --- Bottom Row ---
    \begin{subfigure}[b]{0.48\textwidth}
        \centering
        \includegraphics[width=\linewidth]{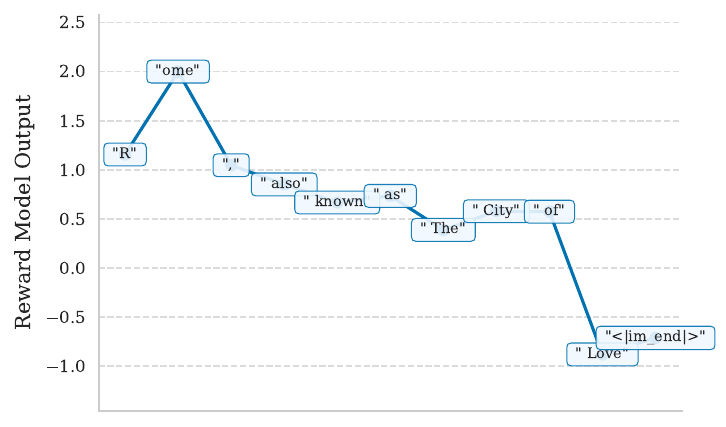} % Placeholder for your third image
        \caption{Wrong suffix 2}
        \label{fig:rome_coherent_wrong_suffix2}
    \end{subfigure}
    \hfill % Adds horizontal space
    
    \caption{Responses to "What is the capital of Italy?" scored by a temporally coherent reward model.}
    \label{fig:rome}
\end{figure}

\subsection{Multi-Step Mathematical Reasoning}

TCRM can be used to evaluate and guide multi-step mathematical reasoning, as established by strong ProcessBench metrics in Section \ref{sec:process_bench}. Here we show an illustrative example of evaluation of step-wise reasoning traces. We use a problem shown in Table \ref{tab:mat_prob}, along with the correct and wrong solutions. Note that it's easy to establish the ground truth for the first wrong step, but there's no clear definition of correctness following a wrong step - you can either follow a correct process and get the wrong numerical values, or use a wrong process and arrive at the correct numerical values.

We evaluate the correctness of steps in the solution by computing the sigmoid of the reward model score at the last token of each step. This sigmoid can be loosely interpreted as the probability of all the steps being correct so far. We expect to see the score staying relatively flat (and at high values) along the correct trajectory. But an error in a step should trigger a large reduction in the score. This is exactly what we see in Figures \ref{fig:math_steps_correct} (scores for the correct solution) and \ref{fig:math_steps_wrong} (scores for the wrong solution). For the wrong solution the scores remain flat in correct Steps 1 and 2, but drop sharply in the erroneous Step 3. Furthermore, Figure \ref{fig:math_tokens_wrong} confirms that at the token resolution our model identifies the exact tokens within Step 3 where the error is introduced.

\begin{table}[h!] % The [h!] suggests LaTeX place the table *here*
\caption{Mathematical problem and two solutions: correct and wrong. The first error in the wrong solution is highlighted in red.}
\label{tab:mat_prob}
\
\centering % Centers the table on the page
\begin{tabular}{p{0.45\textwidth} p{0.45\textwidth}}
\toprule
\multicolumn{2}{c}{\textbf{Problem}} \\
\cmidrule(lr){1-2} % A trimmed rule under the Problem header
\multicolumn{2}{c}{$x+y=3$. Also, $x$ is greater than $y$ by 1. What is $x$?} \\
\midrule
\textbf{Correct Solution} & \textbf{Wrong Solution} \\
\midrule
Let's reason step-by-step: \newline
Step 1. $x$ is greater than $y$ by 1, so $x = y + 1$. \newline
Step 2. $x+y=3$, so $y+1+y=3$, so $2y+1=3$. \newline
Step 3. $2y=2$, so $y=1$. \newline
Step 4. $x=y+1$, so $x=1+1=2$. \newline
Step 5. Therefore, $\boxed{x=2}$. &
Let's reason step-by-step: \newline
Step 1. $x$ is greater than $y$ by 1, so $x = y + 1$. \newline
Step 2. $x+y=3$, so $y+1+y=3$, so $2y+1=3$. \newline
Step 3. $2y=2$, so \textcolor{red}{$y=2$}. \newline
Step 4. $x=y+1$, so $x=2+1=3$. \newline
Step 5. Therefore, $\boxed{x=3}$. \\
\bottomrule
\end{tabular}
\end{table}

\begin{figure}
    \centering
    \includegraphics[width=0.5\linewidth]{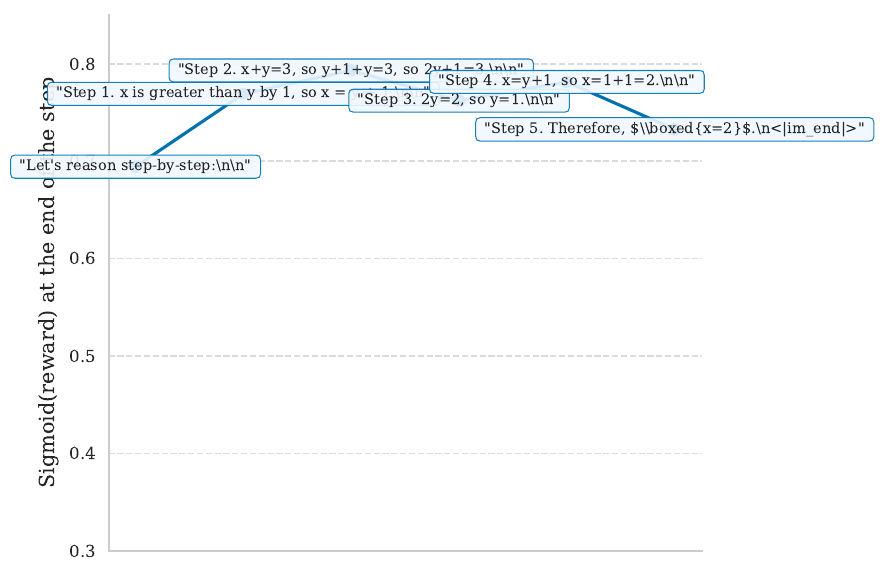}
    \caption{Step-wise reward model scores for the correct solution}
    \label{fig:math_steps_correct}
\end{figure}

\begin{figure}
    \centering
    \includegraphics[width=0.5\linewidth]{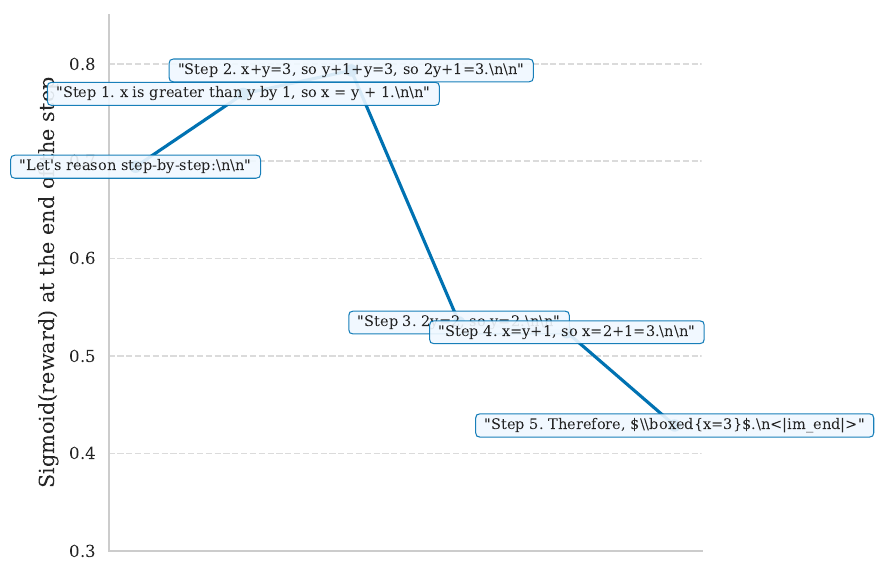}
    \caption{Step-wise reward model scores for the wrong solution}
    \label{fig:math_steps_wrong}
\end{figure}

\begin{figure}
    \centering
    \includegraphics[width=0.5\linewidth]{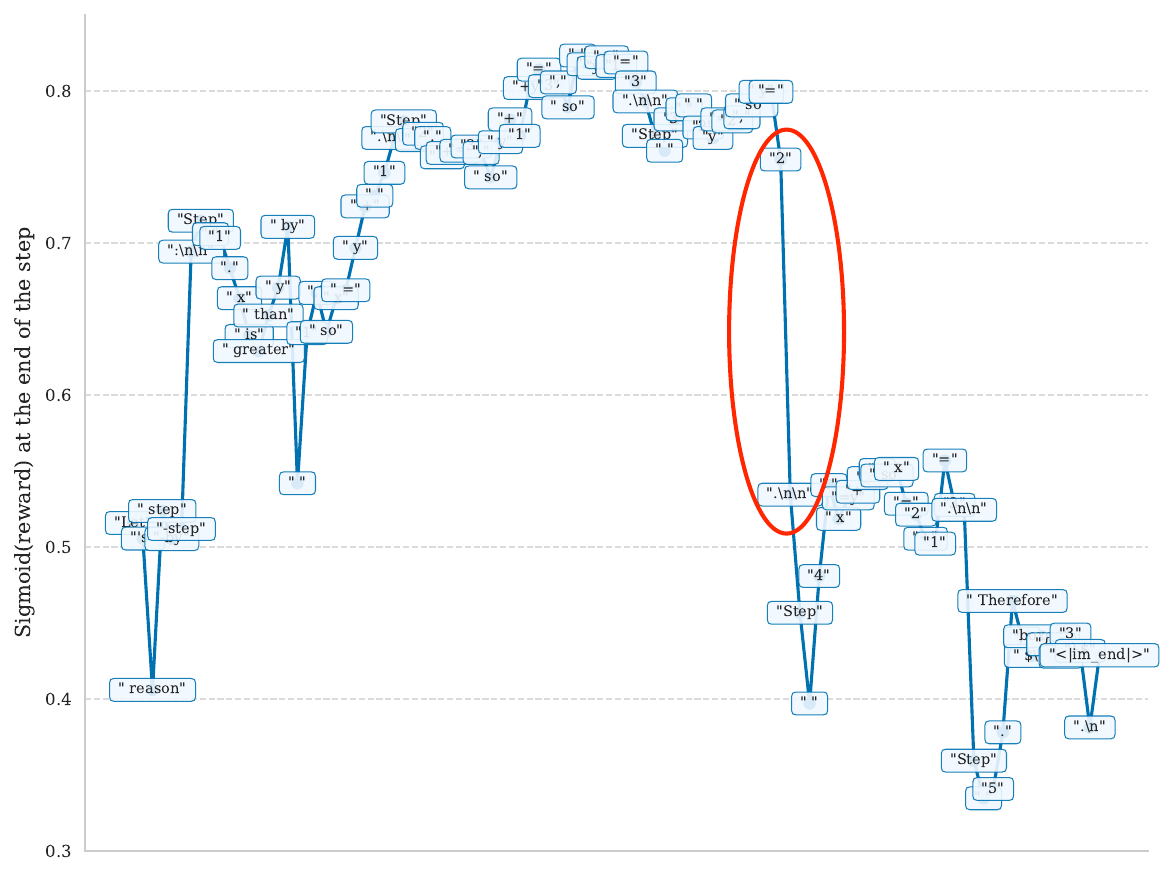}
    \caption{Token-wise reward model scores for the wrong solution}
    \label{fig:math_tokens_wrong}
\end{figure}

\section{ProcessBench Evaluation Details}
\label{sec:pb_eval_details}

Table \ref{tab:process_bench_full} provides the full ProcessBench leaderboard corresponding to the compact main-text table (Table \ref{tab:process_bench}).

\begin{table}[ht]
\centering
\caption{Full PRM evaluation results on ProcessBench. F1 scores of the respective accuracies on erroneous and correct samples are reported.}
\label{tab:process_bench_full}
\begin{tabular}{lrrrrr}
\toprule
\textbf{Model} & \textbf{GSM8K} & \textbf{MATH} & \textbf{Olympiad-Bench} & \textbf{Omni-MATH} & \textbf{Average} \\
\midrule
\multicolumn{6}{c}{Open-weight Process Reward Models (PRMs) trained with \textit{step-wise labels}} \\
\midrule
RLHFlow-Llama3.1-8B-PRM-Mistral & 50.4 & 33.4 & 13.8 & 15.8 & 28.4 \\
RLHFlow-Llama3.1-8B-PRM-Deepseek & 38.8 & 33.8 & 16.9 & 16.9 & 26.6 \\
Skywork-PRM-1.5B & 59.0 & 48.0 & 19.3 & 19.2 & 36.4 \\
Skywork-PRM-7B & 70.8 & 53.6 & 22.9 & 21.0 & 42.1 \\
Qwen2.5-Math-7B-PRM800K & 68.2 & 62.6 & 50.7 & 44.3 & 56.5 \\
Qwen2.5-Math-PRM-7B & 82.4 & 77.6 & 67.5 & 66.3 & 73.5 \\
Qwen2.5-Math-PRM-72B & \textbf{87.3} & \textbf{80.6} & \textbf{74.3} & \textbf{71.1} & \textbf{78.3} \\
\midrule
\multicolumn{6}{c}{Open-weight reward models trained with \textit{outcome labels}} \\
\midrule
Llama3.1-8B-ImplicitPRM-DPO & \textbf{72.1} & 46.0 & 28.0 & 26.7 & 43.2 \\
Llama3.1-8B-ImplicitPRM-CE & 67.6 & 46.2 & 27.5 & 29.7 & 42.8 \\
Llama3.1-8B-ImplicitPRM-DPO (ref-free) & 57.7 & 42.6 & 34.1 & \textbf{36.7} & 42.8 \\
Llama3.1-8B-ImplicitPRM-CE (ref-free) & 51.2 & 39.3 & 32.2 & 33.6 & 39.1 \\
ABC-Llama3.1-8B & 38.5 & 29.3 & 16.0 & 15.1 & 24.7 \\
\textcolor{orange!90!black}{\raisebox{0.1ex}{\large$\bigstar$}} TCRM-Llama3.1-8B (trained on Skywork data) & 30.1 & 24.6 & 21.3 & 18.0 & 23.7 \\
\textcolor{orange!90!black}{\raisebox{0.1ex}{\large$\bigstar$}} TCRM-Llama3.1-8B (trained on math data) & 68.9 & \textbf{47.7} & \textbf{34.8} & 28.3 & \textbf{44.9} \\
\bottomrule
\end{tabular}
\end{table}

For TCRM, ABC and Llama3.1-8B-ImplicitPRM we chose among several methods to classify an individual step as correct or wrong. For each model we independently chose the method and threshold value which maximized the F1 score on GSM8B subset of ProcessBench, in line with the procedure used in \cite{zheng2024processbench}. We considered the following classification methods:

\begin{enumerate}
    \item Difference. We classify the step as correct if the difference of reward model scores at the end of the step (at the token corresponding to the "\string\n \string\n" separator) minus at the end of the previous step was above a threshold. For evaluation of first step we used the score at the first token of the first step for subtraction because there was no previous step.
    \item Sigmoid difference. Similar to the previous method, but sigmoid transformation was applied to the reward scores before subtraction.
    \item Sigmoid ratio. Similar to the previous method, but we used the ratio of sigmoids instead of difference.
\end{enumerate}

Additionally, for reference-free Llama3.1-8B-ImplicitPRM we considered using a "Normalized Difference" method, where the difference between scores was divided by the number of tokens in the step, but this method did not perform well.

For Llama3.1-8B-ImplicitPRM and ABC a cumulative sum was applied to token-level predicted rewards so that we could measure the overall quality of the substring generated so far, instead of individual tokens. Since ABC is a purely inference-time method, we applied it to one of the reward models we had - TCRM-Llama3.1-8B (trained on math data).

\section{RewardBench Results}
\label{sec:rewardbench_v2}

Table \ref{tab:rewardbench_v2} summarizes RewardBench 2 \cite{malik2025rewardbench} results. TCRM consistently outperforms all baselines at larger model scales, achieving the highest overall Score for Qwen3-8B (53.2), Qwen3-32B (74.4), and Llama3.1-8B (68.1), with particularly strong gains in Safety and Focus.
While implicit reward methods like ImplicitPRM (DPO) and TC-$\lambda$ can be competitive at smaller scales (0.6B–1.7B), their performance degrades significantly on larger models---where TCRM maintains or extends its advantage over the standard baseline---suggesting that TCRM scales more effectively with model capacity.

\begin{table}[t]
\centering
\caption{RewardBench v2 results across model sizes and training methods.}
\label{tab:rewardbench_v2}
\resizebox{\textwidth}{!}{%
\begin{tabular}{ll|ccccccc}
\toprule
\textbf{Model} & \textbf{Method} & \textbf{Score} & \textbf{Factuality} & \textbf{Precise IF} & \textbf{Math} & \textbf{Safety} & \textbf{Focus} & \textbf{Ties} \\
\midrule
  \multirow{4}{*}{Qwen3-0.6B} & Baseline & 22.4 & 37.5 & \textbf{28.1} & 25.1 & 21.6 & 16.4 & 5.5 \\
   & ImplicitPRM (DPO) & \textbf{35.5} & \textbf{38.3} & 25.0 & \textbf{38.3} & 32.4 & \textbf{70.9} & 7.9 \\
   & TC-$\lambda$ & 32.8 & 31.2 & 24.4 & 35.0 & \textbf{40.2} & 54.1 & \textbf{12.0} \\
   & TCRM & 29.0 & 30.5 & 26.2 & 32.8 & 29.3 & 50.9 & 4.1 \\
\midrule
  \multirow{4}{*}{Qwen3-1.7B} & Baseline & 28.7 & 27.6 & 23.1 & 40.4 & 44.4 & 30.3 & 6.1 \\
   & ImplicitPRM (DPO) & 42.1 & 46.5 & 21.2 & 50.8 & 52.4 & \textbf{69.1} & \textbf{12.5} \\
   & TC-$\lambda$ & \textbf{44.5} & \textbf{51.0} & \textbf{33.1} & \textbf{56.8} & \textbf{66.4} & 55.0 & 4.9 \\
   & TCRM & 37.2 & 33.7 & 27.5 & 29.5 & 62.0 & 64.4 & 5.9 \\
\midrule
  \multirow{2}{*}{Qwen3-4B} & ImplicitPRM (DPO) & \textbf{51.0} & \textbf{48.8} & 29.4 & \textbf{53.0} & 67.1 & \textbf{80.2} & \textbf{27.4} \\
   & TC-$\lambda$ & 45.4 & 41.5 & \textbf{32.5} & 49.2 & \textbf{74.0} & 60.6 & 14.4 \\
\midrule
  \multirow{4}{*}{Qwen3-8B} & Baseline & 47.4 & 40.2 & 33.8 & 56.3 & 57.8 & 75.6 & 20.9 \\
   & ImplicitPRM (DPO) & 51.7 & 49.1 & 31.9 & 54.6 & 66.0 & \textbf{82.0} & \textbf{26.4} \\
   & TC-$\lambda$ & 48.5 & \textbf{58.7} & 35.0 & 55.7 & 72.7 & 47.5 & 21.1 \\
   & TCRM & \textbf{53.2} & 48.2 & \textbf{37.5} & \textbf{57.4} & \textbf{74.4} & 77.0 & 24.5 \\
\midrule
  \multirow{4}{*}{Qwen3-14B} & Baseline & 72.2 & 67.6 & 42.5 & 70.5 & 90.2 & 80.8 & \textbf{81.5} \\
   & ImplicitPRM (DPO) & 54.1 & 53.9 & 31.2 & 55.7 & 76.2 & 83.8 & 23.9 \\
   & TC-$\lambda$ & \textbf{72.6} & \textbf{67.8} & \textbf{43.1} & \textbf{70.5} & 90.9 & \textbf{85.5} & 77.8 \\
   & TCRM & 71.1 & 66.5 & 41.2 & 65.6 & \textbf{93.6} & 85.1 & 74.5 \\
\midrule
  \multirow{3}{*}{Qwen3-32B} & Baseline & 73.4 & 72.6 & \textbf{38.1} & 68.9 & \textbf{93.1} & 91.1 & 76.4 \\
   & ImplicitPRM (DPO) & 52.0 & 59.4 & 26.2 & 48.1 & 64.9 & 82.4 & 30.8 \\
   & TCRM & \textbf{74.4} & \textbf{73.7} & \textbf{38.1} & \textbf{70.5} & 92.9 & \textbf{93.9} & \textbf{77.3} \\
\midrule
  \multirow{4}{*}{Llama3.1-8B} & Baseline & 67.3 & 59.8 & 33.8 & 65.6 & \textbf{88.4} & 87.5 & 69.0 \\
   & ImplicitPRM (DPO) & 52.5 & 50.9 & 23.8 & 54.1 & 82.7 & 86.1 & 17.5 \\
   & TC-$\lambda$ & 67.6 & \textbf{61.1} & \textbf{35.0} & 59.0 & 88.2 & \textbf{89.1} & 73.1 \\
   & TCRM & \textbf{68.1} & 60.0 & 33.1 & \textbf{67.2} & 87.8 & 87.1 & \textbf{73.3} \\
\bottomrule
\end{tabular}}
\end{table}

\section{Reinforcement Learning Training and Evaluation Details}
\label{sec:rl_details}

We use the Llama 3.1 8B (Instruct) \cite{dubey2024llama3} model as the initial LLM checkpoint for PPO training. The reward models are trained on top of this LLM as well.

\subsection{Training}
\label{sec:rl_training}

The reward models (both baseline and TCRM) are trained on the pairwise dataset \textit{Skywork-Reward-Preference-80K-v0.2} \cite{liu2024skywork}. The prompts for PPO training were sourced from the \textit{Dolci-Instruct-RL} dataset, also used for Olmo 3 \cite{olmo2025olmo3} training. This dataset was split 90/10 into training and validation sets so that we could evaluate the quality of the model during training without overfitting bias.

PPO training was done using the \textit{verl} open-source framework.

The PPO hyperparameters are shown in Table \ref{tab:ppo_hyperparams}. Note that the KL divergence penalty was applied as a loss, not as a reward component. This helped ensure that TCRM can function as a value model.

\begin{table}[h]
    \centering
    \begin{tabular}{|c|c|}
        \hline
        \textbf{Hyperparameter} & \textbf{Value} \\ \hline
        Batch size & 1024 \\ \hline
        PPO mini batch size & 256 \\ \hline
        Maximum prompt length & 1024 tokens \\ \hline
        Maximum response length & 1024 tokens \\ \hline
        Number of training epochs & 1 \\ \hline
        Actor learning rate & 5e-7 \\ \hline
        Critic learning rate & 5e-6 \\ \hline
        KL divergence penalty & 1e-3, loss \\ \hline
        Optimizer & AdamW \\ \hline
        GAE $\gamma$ & 1.0 \\ \hline
        GAE $\lambda$ & 1.0 \\ \hline
        Critic warmup steps & 0 \\ \hline
    \end{tabular}
    \caption{PPO hyperparameters}
    \label{tab:ppo_hyperparams}
\end{table}

\subsection{Evaluation}
\label{sec:rlhf_eval}

For evaluation we used 2 approaches:
\begin{enumerate}
    \item Reward model scoring. This was used to evaluate the checkpoints every 10 training steps to understand training progress.
    \item LLM-as-a-Judge evaluation. This was applied only to generations from 2 final checkpoints (baseline and 1 checkpoint from PPO trained with TCRM RM/VM)
\end{enumerate}

To prevent overfitting bias and signal leakage in reward model scoring we used 2 modifications:
\begin{enumerate}
    \item Split dataset. The prompts for evaluation are also sourced from \textit{Dolci-Instruct-RL}, but we only used the $10\%$ validation subset with no overlap with the PPO training subset. A random subset of 1000 prompts from the validation set was used.
    \item Different reward model. We used the \textit{Skywork-Reward-V2-Llama-3.1-8B} \cite{liu2025skyworkv2} reward model which was trained on the \textit{Skywork-Reward-V2} dataset, which has only small overlap with the \textit{Skywork-Reward-Preference-80K-v0.2} dataset on which the reward models used in PPO were trained.
\end{enumerate}

Figure \ref{fig:llm_judge_prompt} shows the prompt used in LLM-as-a-Judge evaluation. We used Gemini 3 Pro (preview) as the judge model. To mitigate the position bias, each pair of generations was evaluated twice with different order of examples. Evaluations where the LLM's judgment wasn't consistent across different orderings of responses were reported separately in Table \ref{tab:llm_judge}. We sampled 1,000 random prompts for the LLM-as-a-Judge evaluation from the test split of \textit{Dolci-Instruct-RL}.

\begin{figure}[ht]
    \centering
    
    \framebox[\textwidth]{
      \begin{minipage}{0.9\textwidth}
        \setlength{\fboxsep}{5pt}
        You are an expert AI Evaluator. Your task is to analyze a User Query and evaluate two different AI responses (Response A and Response B) to determine which is superior.

\bigskip
\#\#\# EVALUATION CRITERIA

Evaluate based on:

1.  **Instruction Following:** The response must follow all constraints in the prompt, including negative and formatting constraints.

2.  **Accuracy:** The response must be accurate, factual and logically consistent.

3.  **Clarity:** The response must be easy to read and understand.

4.  **Completeness:** The response must fully address all aspects of the prompt.

Provide your judgement as: "A" (Response A is better), "B" (Response B is better), or "TIE" (roughly equal quality).

\bigskip

\#\#\# INPUT DATA

\bigskip

**User Query:**

<user\_query>

\{prompt\}

</user\_query>

\bigskip

**Response A:**

<response\_a>

\{responseA\}

</response\_a>

\bigskip

**Response B:**

<response\_b>

\{responseB\}

</response\_b>

\bigskip
    
    Please evaluate which response is better and provide your judgement.
      \end{minipage}
    }
    % Your original code end
    
    \caption{LLM-as-a-Judge prompt}
    \label{fig:llm_judge_prompt}
\end{figure}

\subsection{Additional Results}
\label{sec:rlhf_add_result}

We get a deeper insight into the training dynamics of PPO with TCRM-based reward and value models by analyzing the KL divergence. Figure \ref{fig:rlhf_kl} shows how KL divergence changes as training progresses and Figure \ref{fig:rlhf_training_progress_kl} shows validation reward model score at different levels of KL divergence. It is clear that PPO with TCRM-based reward and value models uses the "budget" of KL divergence more effectively, achieving higher reward model scores at the same KL divergence. It is notable that a frozen value model based on TCRM achieves substantially lower KL divergence compared to other configurations---we hypothesize that this is due to the reduction of noise from imperfect value model updates.

\begin{figure}
    \centering
    \includegraphics[width=0.75\linewidth]{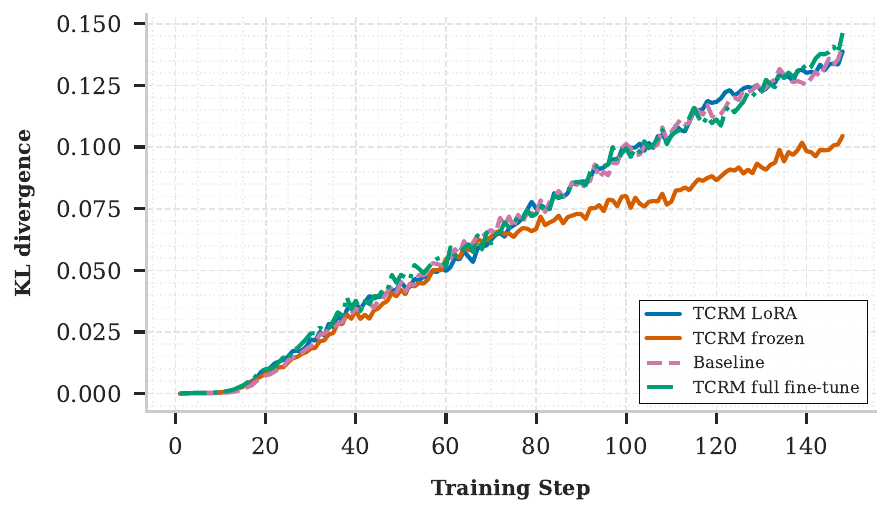}
    \caption{KL Divergence at Different RLHF Training Steps}
    \label{fig:rlhf_kl}
\end{figure}

\begin{figure}
    \centering
    \includegraphics[width=0.75\linewidth]{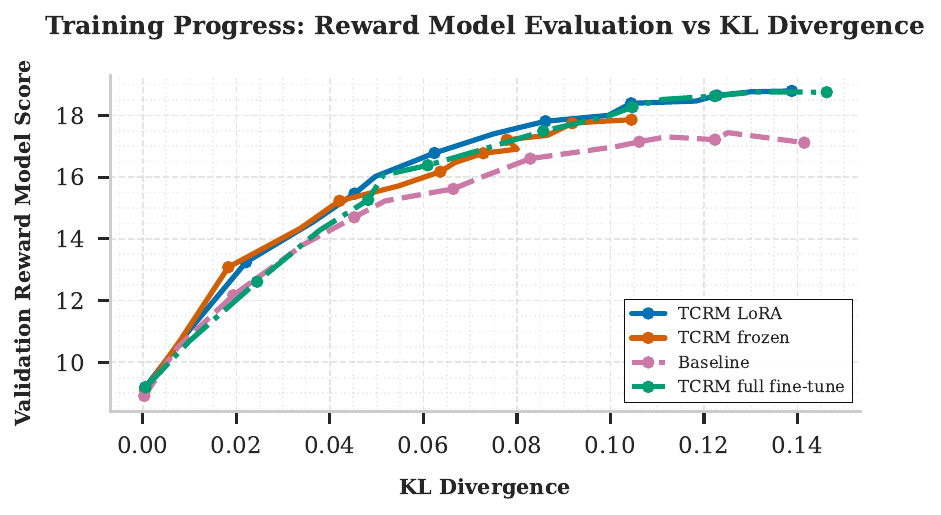}
    \caption{Validation Reward Model Score at Different KL Divergence Levels}
    \label{fig:rlhf_training_progress_kl}
\end{figure}

Initializing a value model from TCRM provides a "warm start" benefit since the value model is accurate from the very first step. We could try to get a similar benefit by freezing the policy LLM until the value model accuracy improves. Figure \ref{fig:rlhf_frozen_critic} shows the result of freezing the policy LLM for 25 steps (this is how long it takes for a value model initialized from the baseline RM to reach the loss similar to the loss of TCRM-based value model). The LLM trained with a frozen value model clearly underperforms the LLMs trained with unfrozen policy LLMs because the first 25 steps essentially waste the data on which the policy LLM could be improved.

\begin{figure}
    \centering
    \includegraphics[width=0.75\linewidth]{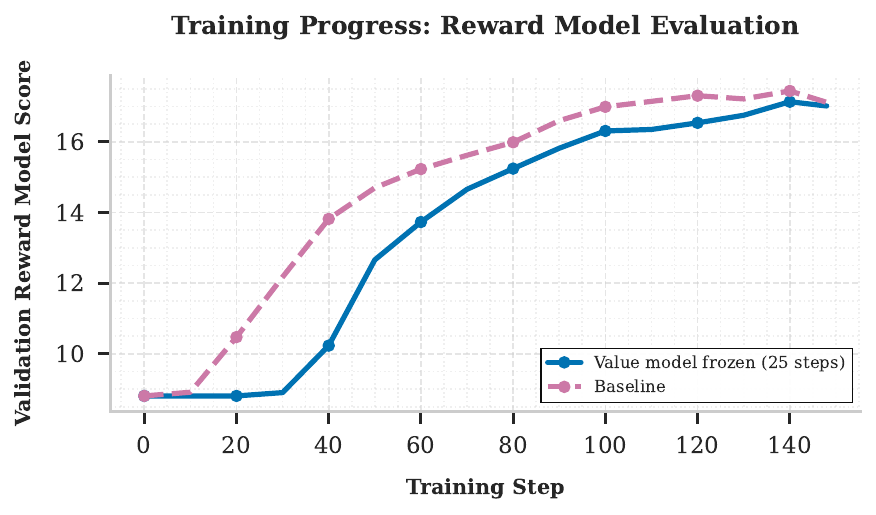}
    \caption{Validation reward model scores for PPO training runs with and without freezing the value model for first 25 steps.}
    \label{fig:rlhf_frozen_critic}
\end{figure}

Using TCRM as a value model for PPO creates an opportunity to optimize resource usage (GPU memory and compute). For most extreme savings we can use a frozen TCRM as both the reward model and the value model, removing the need for a separate value model altogether. Table \ref{tab:tcrm_resource} below shows GPU memory use and step timings for a training setup with 8 H100 GPUs on a single node, using Llama 3.1 8B for policy, reward model and value model. LoRA uses rank-16 adapters. It is important to note that our implementation of memory-efficient TCRM-based PPO is not fully optimized and significant opportunities remain to improve the computational efficiency of PPO with TCRM-based frozen or LoRA value model. With our research implementation we measure a reduction (using frozen TCRM as a value model) in peak GPU memory use of -27\% and a reduction in training step time of -19\%. The LoRA version was not implemented most efficiently, resulting in an increased value model update timing.

\begin{table}[h]
    \centering
    \begin{tabular}{|c|c|c|c|}
        \hline
          & \textbf{Regular PPO} & \textbf{\textcolor{orange!90!black}{\raisebox{0.1ex}{\large$\bigstar$}} TCRM LoRA} & \textbf{\textcolor{orange!90!black}{\raisebox{0.1ex}{\large$\bigstar$}} TCRM Frozen VM} \\ \hline
        Peak GPU allocated memory (GB) & 75.1 & 54.9 (-27\%) & 54.8 (-27\%) \\ \hline
        Training step time (s) & 66.4 & 67.7 (+2\%) & 54.1 (-19\%) \\ \hline
        Value model update time (s) & 12.4 & 13.0 (+5\%) & 0.0 (-100\%)\\ \hline
    \end{tabular}
    \caption{Resource Requirement Comparison: Regular PPO vs TCRM-based Efficient Implementation}
    \label{tab:tcrm_resource}
\end{table}

\end{document}